\providecommand{\pgfsyspdfmark}[3]{}
\PassOptionsToPackage{usenames,dvipsnames}{xcolor}
\PassOptionsToPackage{capitalize}{cleveref}

\documentclass[final,12pt,cleveref]{colt2024-arxiv}

\usepackage{times}
\usepackage{algorithm}
\usepackage{algorithmic}
\usepackage{wrapfig}

\usepackage[utf8]{inputenc}
\usepackage[nolist,nohyperlinks]{acronym}
\usepackage{multirow}
\usepackage{enumitem}

\usepackage{dsfont}
\usepackage{comment}

\usepackage[colorinlistoftodos, shadow,color=blue!30!white
					,disable
]{todonotes}
\usepackage{cancel}
\renewcommand{\Pr}{\field{P}}

\newcommand{\sX}{\mathcal{X}}

\newcommand{\field}[1]{\mathbb{#1}}

\newcommand{\R}{\field{R}}

\newcommand{\E}{\field{E}}

\newcommand{\dt}{\displaystyle}

\newcommand{\reals}{\mathbb{R}}

\DeclareMathOperator{\Regret}{Regret}
\DeclareMathOperator{\Wealth}{W}

\newcommand{\KL}[2]{\operatorname{KL}\left({#1};{#2}\right)}

\usepackage{xspace}

\def\ddefloop#1{\ifx\ddefloop#1\else\ddef{#1}\expandafter\ddefloop\fi}
\def\ddef#1{\expandafter\def\csname c#1\endcsname{\ensuremath{\mathcal{#1}}}}
\ddefloop ABCDEFGHIJKLMNOPQRSTUVWXYZ\ddefloop
\def\ddef#1{\expandafter\def\csname b#1\endcsname{\ensuremath{{\boldsymbol{#1}}}}}
\ddefloop ABCDEFGHIJKLMNOPQRSUVWXYZabcdefghijklmnopqrtsuvwxyz\ddefloop
\def\ddef#1{\expandafter\def\csname h#1\endcsname{\ensuremath{\hat{#1}}}}
\ddefloop ABCDEFGHIJKLMNOPQRSTUVWXYZabcdefghijklmnopqrsuvwxyz\ddefloop \def\ddef#1{\expandafter\def\csname hc#1\endcsname{\ensuremath{\widehat{\mathcal{#1}}}}}
\ddefloop ABCDEFGHIJKLMNOPQRSTUVWXYZ\ddefloop
\def\ddef#1{\expandafter\def\csname bar#1\endcsname{\ensuremath{\bar{#1}}}}
\ddefloop ABCDEFGHIJKLMNOPQRSTUVWXYZabcdefghijklmnopqrtsuvwxyz\ddefloop
\def\ddef#1{\expandafter\def\csname wbar#1\endcsname{\ensuremath{\overline{#1}}}}
\ddefloop ABCDEFGHIJKLMNOPQRSTUVWXYZabcdefghijklmnopqrtsuvwxyz\ddefloop
\def\ddef#1{\expandafter\def\csname tc#1\endcsname{\ensuremath{\widetilde{\mathcal{#1}}}}}
\ddefloop ABCDEFGHIJKLMNOPQRSTUVWXYZ\ddefloop

  \DeclareMathOperator{\EE}{\mathbb{E}}

\DeclareMathOperator{\PP}{\mathbb{P}}

\DeclareMathOperator{\one}{\mathds{1}\hspace{-.18em}}

\newcommand{\fr}[2]{ { \frac{#1}{#2} }}

\newcommand*\diff{\mathop{}\!\mathrm{d}}
\def\cd{\cdot}
\def\dt{{\ensuremath{\delta}\xspace} }
\def\KL{\ensuremath{\normalfont{\mathsf{KL}}}}
\def\ZCP{\ensuremath{\normalfont{\mathsf{ZCP}}}}

\newcommand{\pare}[1]{\left( #1 \right)}

\newcommand{\df}{=}

\usepackage{commath}
\let\cbr\undefined
\providecommand{\kjcbr}[2][-1]{
\ensuremath{\mathinner{
\ifthenelse{\equal{#1}{-1}}{ \left\{#2\right\}}{}
\ifthenelse{\equal{#1}{0}}{ \{#2\}}{}
\ifthenelse{\equal{#1}{1}}{ \bigl\{#2\bigr\}}{}
\ifthenelse{\equal{#1}{2}}{ \Bigl\{#2\Bigr\}}{}
\ifthenelse{\equal{#1}{3}}{ \biggl\{#2\biggr\}}{}
\ifthenelse{\equal{#1}{4}}{ \Biggl\{#2\Biggr\}}{}
}} }  \let\cbr=\kjcbr

\providecommand{\onec}[2][-1]{
\ensuremath{\mathinner{
\one\cbr[#1]{#2}
} }  }

\def\kl{{\mathsf{kl}}}

\def\th{\theta}
\def\Th{\Theta}

\def\RR{{\mathbb{R}}}
\def\NN{{\mathbb{N}}}
\def\TV{{\normalfont\text{TV}}}
\def\KL{{\normalfont\text{KL}}}
\def\ZCP{{\normalfont\text{ZCP}}}
\def\zcp{{\normalfont\text{ZCP}}}

\usepackage{algorithmic}

\def\dP{{\dif P}}
\def\dQ{{\dif Q}}

\def\sm{\setminus} \begin{acronym}
\acro{RLS}{Regularized Least Squares}
\acro{ERM}{Empirical Risk Minimization}
\acro{RKHS}{Reproducing kernel Hilbert space}
\acro{DA}{Domain Adaptation}
\acro{PSD}{Positive Semi-Definite}
\acro{SGD}{Stochastic Gradient Descent}
\acro{OGD}{Online Gradient Descent}
\acro{GD}{Gradient Descent}
\acro{SGLD}{Stochastic Gradient Langevin Dynamics}
\acro{IW}{Importance Weighted}
\acro{MGF}{Moment-Generating Function}
\acro{ES}{Efron-Stein}
\acro{ESS}{Effective Sample Size}
\acro{KL}{Kullback-Leibler}
\acro{SVD}{Singular Value Decomposition}
\acro{PL}{Polyak-Łojasiewicz}
\acro{NTK}{Neural Tangent Kernel}
\acro{KLS}{Kernelized Least-Squares}
\acro{KRLS}{Kernelized Regularized Least-Squares}
\acro{ReLU}{Rectified Linear Unit}
\acro{NTRF}{Neural Tangent Random Feature}
\acro{NTF}{Neural Tangent Feature}
\acro{RF}{Random Feature}
\acro{RWY}{Raskutti-Wainwright-Yu}
\acro{CW}{Celisse-Wahl}
\acro{PRM}{Penalized Reward Maximization}
\acro{ONS}{Online Newton Step}
\acro{ZCP}{Zhang-Cutkosky-Paschalidis}
\acro{KT}{Krichevsky–Trofimov}
\end{acronym}

\title{Better-than-KL PAC-Bayes Bounds}

\coltauthor{\Name{Ilja Kuzborskij} \Email{iljak@google.com}\\
 \addr Google DeepMind\AND
 \Name{Kwang-Sung Jun} \Email{kjun@cs.arizona.edu}\\
 \addr University of Arizona\AND 
 \Name{Yulian Wu} \Email{yulian.wu@kaust.edu.sa}\\
 \addr KAUST\AND
 \Name{Kyoungseok Jang} \Email{ksajks@gmail.com}\\
 \addr New York University\AND 
 \Name{Francesco Orabona} \Email{francesco@orabona.com}\\
 \addr KAUST}

\begin{document}

\maketitle

\begin{abstract}Let $f(\theta, X_1),$ $ \dots,$ $ f(\theta, X_n)$ be a sequence of random elements, where $f$ is a fixed scalar function, $X_1, \dots, X_n$ are independent random variables (data), and $\theta$ is a random parameter distributed according to some data-dependent \emph{posterior} distribution $P_n$.
  In this paper, we consider the problem of proving concentration inequalities to estimate the mean of the sequence.
  An example of such a problem is the estimation of the generalization error of some predictor trained by a stochastic algorithm, such as a neural network where $f$ is a loss function.
  Classically, this problem is approached through a \emph{PAC-Bayes} analysis
  where, in addition to the posterior, we choose a \emph{prior} distribution which captures our belief about the inductive bias of the learning problem.
  Then, the key quantity in PAC-Bayes concentration bounds
  is a divergence that captures the \emph{complexity} of the learning problem where the de facto standard choice is the \ac{KL} divergence.
  However, the tightness of this choice has rarely been  questioned.
  
  In this paper, we challenge the tightness of the KL-divergence-based bounds by showing that it is possible to achieve a strictly tighter bound.
  In particular, we demonstrate new \emph{high-probability} PAC-Bayes bounds with a novel and \emph{better-than-KL} divergence that is inspired by Zhang et al. (2022).
  Our proof is inspired by recent advances in regret analysis of gambling algorithms, and its use to derive concentration inequalities.
  Our result is first-of-its-kind in that existing PAC-Bayes bounds with non-KL divergences are not known to be strictly better than KL.
  Thus, we believe our work marks the first step towards identifying optimal rates of PAC-Bayes bounds.
\end{abstract}


\section{Introduction}
\label{sec:intro}

We study the standard model of statistical learning, where we are given a sample of independent observations $X_1, \ldots, X_n \in \cX$,
and we have access to a measurable parametric function $f : \Th \times \cX \to [0,1]$.
In particular, we are interested in estimating the mean of $f$ when $\th$ is a \emph{random} parameter drawn from a distribution chosen by an algorithm based on data (typically called the \emph{posterior} $P_n$).
In other words, our goal is to estimate the mean\footnote{In the following, the integration $\int \equiv \int_{\Th}$ is always understood w.r.t. $\th \in \Th$ while the expectation $\E$ is always taken w.r.t. data.}
\begin{align*}
  \int \E[f(\th, X_1)] \diff P_n(\th)~.
\end{align*}
In many learning scenarios, $f(\th, X_i)$ is interpreted as a composition of a loss function and a predictor evaluated on example $X_i$, given the parameter $\th \sim P_n$.
Often, this problem is captured by giving bounds on the \emph{generalization error} (sometimes called \emph{generalization gap})
\begin{align*}
  \int \Delta_n(\th) \diff P_n(\th), \qquad \text{where} \qquad
  \Delta_n(\th) = \frac1n \sum_{i=1}^n (f(\th, X_i) - \E[f(\th, X_1)])~.
\end{align*}

To have a sharp understanding of the behavior of the generalization error one often desires to have bounds that hold with high probability over the sample.
At the same time, standard tools for this task, such as concentration inequalities (for instance, Chernoff or Hoeffding inequalities) are not applicable here, since $P_n$ itself depends on the sample and can be potentially supported on infinite sets (which precludes the use of union bounds).
In the particular setting of a randomized prediction discussed here,
studying such bounds was a long topic of interest in the
\emph{PAC-Bayes} analysis~\citep{shawetaylor1997pac,McAllester98}.
PAC-Bayes bounds typically require an additional component called the \emph{prior} distribution over parameter space, which captures our belief about the inductive bias of the problem.
Then, for any data-free prior distribution $P_0$, a classical result holds with probability at least $1-\delta$ (for a failure probability $\delta \in (0,1)$) over the sample, \emph{simultaneously} over all choices of data-dependent posteriors $P_n$:
\begin{align}
  \label{eq:mcallester}
  \int \Delta_n(\th) \diff P_n(\th) 
  = \cO\pare{ \sqrt{\frac{D_{\KL}(P_n, P_0) + \ln \frac{\sqrt{n}}{\delta} }{n}} } \quad \text{as} \quad n \to \infty~.
\end{align}
While such a bound is uniform over all choices of $P_n$, it does scale with a \ac{KL} divergence between them, $D_{\KL}(P_n,P_0) = \int \ln(\diff P_n/ \diff P_0) \diff P_n$,  which can be thought of as a measure of complexity of the learning problem.
Over the years, PAC-Bayes bounds were developed in many ever tighter variants, such as the ones for Bernoulli losses~\citep{LangfordCaruana2001,seeger2002,maurer2004note}, exhibiting fast-rates given small loss variances  \citep{tolstikhin2013pac,MhammediGG19}, data-dependent priors~\citep{rivasplata2020pac,awasthi2020pac}, and so on.
However, one virtually invariant feature remained: high probability PAC-Bayes bounds were always stated using the $\KL$-divergence.
The reason is that virtually all of these proofs were based on the Donsker-Varadhan \emph{change-of-measure} inequality\footnote{For any measurable $F$, and any $P_n,P_0$, the inequality states that $\int F \diff P_n \leq D_{\KL}(P_n, P_0) + \ln \int e^F \diff P_0$.} (essentially arising from a relaxation of a variational representation of $\KL$ divergence)~\citep{DoVa75}.

Recently several works have looked into PAC-Bayes analyses arising from the use of different change-of-measure arguments~\citep{begin2016pac,alquier2018simpler,ohnishi2021novel},
allowing to replace $\KL$ divergence with other divergences such as $\chi^2$ or Hellinger,
however these results either did not hold with high probability or involved looser divergences (such as $\chi^2$).

In this work we propose an alternative change-of-measure analysis and show, to the best of our knowledge, the first high-probability PAC-Bayes bound that involves a divergence that is \emph{tighter} than $\KL$-divergence.

Our analysis is inspired by a recent observation of \cite{zhang2022optimal}, who pointed out an interesting phenomenon arising in regret analysis of online algorithms~\citep{cesa2021online}: They focused on a classical problem of learning with experts advice~\citep{Cesa-BianchiL06}, where the so-called parameter-free regret bounds scale with the $\sqrt{D_{\KL}}$ between the competitor distribution over experts and the choice of the prior~\citep{LuoS15,OrabonaP16}.
In particular, their analysis
improves over parameter-free rates replacing $\sqrt{D_{\KL}}$ by a divergence with the shape\begin{align}\label{eq:zcp-divergence}
  D_\zcp(P_n, P_0; c) = \int \abs{\fr{\diff P_n}{\diff P_0}(\th) - 1} \sqrt{\ln\pare{1 + c^2 \, \abs{\fr{\diff P_n}{\diff P_0}(\th) - 1}^2 }} \diff P_0(\th), 
\end{align}
where $c$ is a parameter.\footnote{
  There is one minor difference that the original divergence appeared in \citet{zhang2022optimal} as $\sqrt{\ln(1 + |\frac{\diff P_n}{\diff P_0}(\th)-1|)}$ instead of $\sqrt{\ln(1 + c^2|\frac{\diff P_n}{\diff P_0}(\th)-1|^2)}$.
}
We call it \ac{ZCP} divergence.
Interestingly, the \ac{ZCP} divergence enjoys the following upper bound (\Cref{cor:zcp-to-tv-kl}):
\begin{align}
  \label{eq:zcp-to-kl-tv}
  D_\zcp = \tilde{\cO}\big(\sqrt{D_{\KL} \, D_{\TV}} + D_{\TV}\big),
\end{align}
where $D_{\TV}(P_n, P_0) = \frac12 \int |\diff P_n - \diff P_0|$ is the total variation (TV) distance.
Since $D_{\TV} \leq 1$ for any pair of distributions, $D_\zcp$ is orderwise tighter than $\sqrt{D_{\KL}}$ and we show in \cref{sec:bernoulli-advantage} that in some cases the gain can be substantial.

In this paper, we develop a novel and straightforward change-of-measure type analysis that leads to PAC-Bayes bounds with \ac{ZCP} divergence, avoiding the regret analysis of \citet{zhang2022optimal} altogether.

\paragraph{Our contributions}
Our overall contribution is to show a surprising result that the choice of the KL divergence as the complexity measure in PAC-Bayes bounds is suboptimal, which tells us that there is much room for studying optimal rates of PAC-Bayes bounds.

Specifically, we show that the KL divergence of existing PAC-Bayes bounds can be strictly improved using a different, better divergence. 
We achieve it through two main results.
Our first result is a PAC-Bayes bound (\Cref{thm:zcp-hoeffding})
\begin{align*}
\int \Delta_n(\th) \diff P_n(\th) \leq \frac{\sqrt{2} \, D_{\zcp}(P_n, P_0; \sqrt{2 n} / \delta) + 2 + \sqrt{\ln\frac{2 \sqrt{n}}{\delta}}}{\sqrt{n}},
\end{align*}
which readily improves upon \cref{eq:mcallester}, since by upper-bounding a $\ZCP$-divergence we have
\begin{align}
  \label{eq:intro-zcp-bound}
  \int \Delta_n(\th) \diff P_n(\th) = \cO\Bigg(
  \sqrt{\frac{D_{\KL}(P_n, P_0) \, D_{\TV}(P_n, P_0) + \ln \frac{\sqrt{n}}{\delta}}{n}}
  \Bigg)
  \quad \text{as} \quad n \to \infty~.
\end{align}

Our second contribution is a bound that extends to regimes beyond $1/\sqrt{n}$ rate, such as fast rates of order $1/n$ when the sample variance of $f$ is small.
Here, we consider variants of empirical Bernstein inequality and a bound on `little-kl' (a bound for binary or bounded $f()$), see \Cref{sec:recovering-bounds}.
In fact, instead of deriving individual bounds separately, we first show a generic bound that can be relaxed to obtain these cases.
This technique is inspired by the recent advances on obtaining concentration inequalities through the regret analysis of online betting algorithms~\citep{JunO19, orabona2023tight}. 
In particular, we consider the expected  \emph{optimal log wealth} (denoted by $\ln \Wealth_n^*$) of an algorithm that bets on the outcomes $f(\th, X_i) - \E[f(\th, X_1)]$.
Then, using the regret bound of this algorithm (which is only required for the proof), we obtain concentration inequalities from upper bounds on $\int \cdot \dP_n$ of 
\begin{align*}
  \ln \Wealth_n^*(\th) = \max_{\beta \in [-1,1]}\ \ln \prod_{i=1}^n \pare{1 + \beta (f(\th, X_i) - \E[f(\th, X_1)])}~.
\end{align*}
Recently \citet{jang2023tighter} have used this technique to control $\int \ln \Wealth_n^* \dP_n$ in terms of $\KL$ divergence, and used various lower bounds on the logarithm
to recover many known PAC-Bayes inequalities, such as PAC-Bayes Bernstein inequality and others.
As an illustrative example, consider the simple inequality $\ln(1+x) \geq x - x^2$ for $|x| \leq 1/2$:
If one can show that $\int \ln \Wealth_n^* \dP_n \leq \mathrm{bound}(\delta, n, P_n, P_0)$, then the above implies that $\max_{|\beta|\le 1/2}\{ \beta \Delta_n - \beta^2 n\} \leq \mathrm{bound}(\delta, n, P_n, P_0)$ and so we can optimally tune $\beta$ to obtain $\int \Delta_n \dP_n \leq \sqrt{\mathrm{bound}(\delta, n, P_n, P_0) / n}$, which recovers a familiar bound in the shape of \cref{eq:mcallester}.
A more fine-grained analysis leads to an empirical Bernstein type inequality (see Corollary~\ref{cor:empiricalbernstein}). This suggests that the optimal log wealth $\ln \Wealth_n^*$ is the fundamental quantity that unifies various existing types of concentration bounds such as Hoeffding, Bernoulli-KL, and empirical Bernstein inequalities.

Using the above concept of optimal log wealth, we show (see \cref{eq:asymptotic-bound}) that there exists a universal constant $c > 0$, such that almost surely for all distributions $P$ (possibly data-dependent ones),
\begin{align*}
  \limsup_{n \to \infty} \ \frac{\int \ln \Wealth^*_n(\th) \dP(\th)}{\ln^{3/2}(n)} \leq c \, \pare{
  1 + \sqrt{D_{\KL}(P, P_0) \, D_{\TV}(P, P_0)} \, \big(1 + \sqrt{D_{\KL}(P, P_0)}\big)
  }~.
\end{align*}
We state this asymptotic bound in terms of the upper bound on the $\ZCP$ divergence.
However,
compared to \cref{eq:intro-zcp-bound}, this bound is more versatile as it can be used to obtain an empirical Bernstein inequality.
That is, we show later that it implies
\[
  \limsup_{n\to \infty} \ \frac{\abs{\int \Delta_n(\th) \diff P}^2}{\frac1n \pare{\abs{\int \Delta_n(\th) \diff P}  + \hat{V}(P)} \cdot \ln^{3/2}(n)}
  \leq
  c \, \pare{
    1 + \sqrt{D_{\KL} \, D_{\TV}} \, \big(1 + \sqrt{D_{\KL}}\big)
    }~.  
  \]
While it comes with an additional dependency on $\sqrt{D_{\KL}}$, we show later in \cref{sec:divergence} that is still never worse than existing KL-based bounds yet enjoys orderwise better bounds in some instances.

\section{Additional related work}
\paragraph{PAC-Bayes}
\emph{PAC-Bayes} has been a long-lasting topic of interest in statistical learning~\citep{shawetaylor1997pac,McAllester98,catoni2007}, with a considerably interest in both theory and applications; see \citet{alquier2024user} for a comprehensive survey.
Over the years, most of the PAC-Bayes literature was concerned with tightening \cref{eq:mcallester} by using more advanced concentration inequalities and making assumptions about data-generating distributions.
A notable improvement is a
bound that switches to the rate $1/n$ for a sufficiently small sample variance~\citep{tolstikhin2013pac}, an empirical Bernstein inequality, which was further improved by \citet{MhammediGG19}.
Fast rates where also noticed in other bounds, which are useful in situations when losses are sufficiently small~\citep{catoni2007,yang2019fast}.
Several results have also relaxed the independence assumption in PAC-Bayes analysis
through martingale conditions~\citep{seldin2012pac,kuzborskij2019efron,haddouche2022pac,lugosi2023online}.
Finally, a recent surge of practical interest in PAC-Bayes was stimulated by its ability to yield numerically non-vacuous generalization bounds for deep neural networks~\citep{dziugaite2017computing,perez-ortiz2020tighter,dziugaite2021role}.

All of these results, similarly as the classical ones, involve the $\KL$ divergence, which is considered as the de facto standard divergence to be used for PAC-Bayes bounds.
Our paper indicates that this standard KL-based bound is in fact suboptimal by showing that there exists a new divergence that is never worse than KL orderwise while showing strict improvements in special cases.
\paragraph{Other divergences and connection to change-of-measure inequalities}
The focus of this paper is on the `complexity' term, or divergence between the posterior and prior distributions over parameters.
Several works have dedicated some attention to this topic and obtained bounds with alternative divergences.
\citet{alquier2018simpler} studied the setting of unbounded losses and
derived PAC-Bayes bounds with $\chi^2$ divergence instead of $\KL$ divergence, however at the cost of a high probability guarantee: their bound scales with $1/\delta$ rather than $\ln(1/\delta)$.
In another notable work, \citet{ohnishi2021novel} proved a suite of change-of-measure inequalities distinct from the usual Donsker-Varadhan inequality.
Their method is based on a tighter variational representation of $f$-divergences developed by ~\citet{ruderman2012tighter}, which in turn improves upon~\citet{nguyen10estimating}.
The variational representation arises from the use the Fenchel-Young inequality with respect to $f$ under the integral operator.
For example, for some measurable function $F$, we have
$\int \frac{\dP}{\dQ} \cdot F \dQ \leq \int \big( f(\frac{\dP}{\dQ}) + f^{\star}(F) \big) \dQ$
where $f^{\star}$ is a convex conjugate of $f$.
In this paper, we focus on a particular $f$, whose convex conjugate is a function $f^{\star}(y) = \delta \, \exp(y^2 / (2 n))$.
In fact, $\int f(\frac{\dP}{\dQ} - 1) \dQ$ then gives rise to the \ac{ZCP} divergence.

Interestingly, the function of a shape $y \mapsto \exp(y^2 / 2)$ appears in several other contexts.
In the online learning literature this function, identified as the potential function or the dual of the regularizer, is used in the design and analysis of the so-called parameter-free algorithms, both for learning with expert advice~\citep{chaudhuri2009parameter, LuoS15, koolen2015second, OrabonaP16} and for online convex optimization~\citep{OrabonaP16}.

\citet{chu2023unified} also derived an interesting generalization error bound using the Fenchel-Young inequality in the $L_p$-Orlicz norm. In their analysis they focus on the function $f^\star(y)=\exp(y^p)$, with the majority of their results obtained with $p=2$.
Albeit their analysis commences from the Fenchel-Young inequality,
later on it is simplified through the use of the inverse function $f^{-1}$ instead of the dual $f^{\star}$, resulting in a looser upper bound, ultimately leading back to the KL divergence.

\paragraph{Concentration from coin-betting}
Our paper occasionally relies on the coin-betting formalism (see \Cref{sec:betting}), which goes back to \citet{Ville39} and the Kelly betting system~\citep{Kelly56}.
The coin-betting formalism can be thought of as a simple instance  of the Universal Portfolio theory~\citep{Cover91}.
The idea of showing concentration inequalities using regret of online betting algorithms was first investigated by \citet{JunO19} who established new (time-uniform) concentration inequalities.
Their work was heavily inspired by \citet{RakhlinS17} who showed the equivalence between the regret of generic online linear algorithms and martingale tail bounds. However, the idea of linking online betting to statistical testing, but without the explicit use of regret analysis, goes back at least to \citet{Cover74} and more recently to \citet{ShaferV05}.

\section{Definitions and preliminaries}
\label{sec:def}
We denote by $(x)_+ = \max\{x, 0\}$.
Let $P$ and $Q$ be probability measures over $\Theta$.
For a convex function $f : [0, \infty) \to (-\infty, \infty]$ such that $f(x)$ is finite for all $x > 0$, $f(1) = 0$, and $f(0) = \lim_{x \to 0^+} f(x)$ (possibly infinite),
the \emph{$f$-divergence} between $P, Q \in \cM_1(\Th)$ is defined as $D_f(P, Q) = \int_{\Th} f(\diff P/\diff Q) \diff Q$.
In particular, we will encounter several $f$-divergences such as $\KL$-divergence $D_{\KL}(P, Q) \df \int \ln\big(\frac{\dP(\th)}{\dQ(\th)}\big) \, \dP(\th)$ and
$\TV$-distance $D_{\TV}(P, Q) = \frac12 \int |\frac{\dP(\th)}{\dQ(\th)} - 1| \dQ(\th)$.

We will also encounter a \emph{R\'enyi divergence} of order $\alpha \in (0,1) \cup (1, \infty)$, defined as $D_{\alpha}(P, Q) = \frac{1}{\alpha - 1} \, \ln \int \dP(\th)^{\alpha} \dQ(\th)^{1-\alpha}$, which has many interesting connections to $f$-divergences, as discussed in \citet{van2014renyi}.
For instance, $\lim_{\alpha \to 1} D_{\alpha} = D_{\KL}$, while $D_{2} = \ln(1 + D_{\chi^2})$.

If a set $\sX$ is uniquely equipped with a $\sigma$-algebra, the underlying $\sigma$-algebra will be denoted by $\Sigma(\sX)$.
We formalize a `data-dependent distribution' through the notion of \emph{probability kernel}~\citep[see, e.g.,][]{kallenberg2017}, which is defined as a map $K : \sX^n \times \Sigma(\Theta) \rightarrow [0,1]$ such that
for each $B \in \Sigma(\Theta)$ the function $s \mapsto K(s,B)$ is measurable and for each $s \in \sX^n$ the function $B \mapsto K(s,B)$ is a probability measure over $\Theta$.
We write $\cK(\sX^n,\Theta)$ to denote the set of all probability kernels from $\cX^n$ to distributions over $\Theta$.
In that light, when $P \in \cK(\sX^n,\Theta)$ is evaluated on $S \in \sX^n$ we use the shorthand notation $P_n = P(S)$.

\subsection{Coin-betting game, regret, and Ville's inequality}
\label{sec:betting}
The forthcoming analysis is intimately connected to the derivation of concentration inequalities through regret analysis of online gambling algorithms, following \citet{ JunO19, orabona2023tight}.
Here, we briefly introduce the required notions and definitions.

We consider a gambler playing a betting game repetitively. This gambler starts with initial wealth $\Wealth_0 \df 1$. In each round $t$, the gambler bets $|x_t|$ money on the outcome $x_t$, observes a `continuous coin' outcome $c_t \in [-1,1]$, which can even be adversarially chosen. At the end of the round, the gambler earns $x_t c_t$, so that if we define the wealth of the gambler at time $t$ as $\Wealth_t$, 
\[
\Wealth_t \df \Wealth_{t-1} + c_t x_t= 1 + \sum_{s=1}^t c_s x_s~.
\]
We also assume the gambler cannot use any loans in this game, meaning $\Wealth_t\geq 0$ and $x_t \in [-\Wealth_{t-1}, \Wealth_{t-1}]$.

This coin-betting game is one of the representative online learning problems. Therefore, comparing the difference in wealth with a fixed strategy, as in online betting, is a natural objective. In particular, for each $\beta\in[-1,1]$, let $\Wealth_t(\beta)$ be the wealth when the gambler bets $\beta \Wealth_{t-1}(\beta)$ in round $t$ with the same initial wealth condition $\Wealth_0(\beta)\df 1$. So, we have\[    \Wealth_t(\beta)    \df \Wealth_{t-1}(\beta) + c_t \beta \Wealth_{t-1}(\beta)    = \prod_{s=1}^t (1+c_s \beta)~.    \]We denote $\Wealth^*_n= \max_{\beta \in [-1,1]} \ \Wealth_n(\beta)$ the maximum wealth a fixed betting strategy can achieve, and define the regret of the betting algorithm as     \[    \Regret_n    \df \frac{\Wealth^*_{n}}{\Wealth_n}~.\]
It is well-known that it is possible to design optimal online betting algorithms with regret bounds that are polynomial in $n$~\citep[Chapters 9 and 10]{Cesa-BianchiL06}.

If the coin outcomes have conditional zero mean, it is intuitive that any online betting algorithm should not be able to gain any money.
Indeed, the wealth remains 1 in expectation (i.e., $\EE[\Wealth_t] = 1$), so $\Wealth_t$ forms a nonnegative martingale and thus follows a very useful time-uniform concentration bound known as Ville's inequality.
\begin{theorem}[{Ville's inequality~\citep[p.~84]{Ville39}}]
  \label{thm:ville}
  Let $Z_1, \ldots, Z_n$ be a sequence of non-negative random variables such that $\E[Z_i \mid Z_1, \ldots, Z_{i-1}] = 0$.
  Let $M_t > 0$ be $\Sigma(Z_1, \dots, Z_{t-1})$-measurable such that $M_0 = 1$, and moreover assume that $\E[M_t \mid Z_1, \dots, Z_{t-1}] \leq M_{t-1}$. 
  Then, for any $\delta \in (0, 1]$,
  $
    \PP\left\{\max_{t} M_t \geq \frac{1}{\delta}\right\} \leq \delta
  $.
\end{theorem}
Ville's inequality will be the main tool to leverage regret guarantees to construct our concentration inequalities.

\section{The ZCP Divergence}
\label{sec:divergence}
Here, we look deeper into properties of the \ac{ZCP} divergence  defined in \cref{eq:zcp-divergence}. 
First, note that $\ZCP$-divergence is an $f$-divergence with $f(x) = |x-1| \sqrt{\ln(1+c^2 \, |x-1|^2)}$ for $x \in \R_{\geq 0}$ and some parameter $c \geq 0$.
The main interesting property of this divergence is that it is controlled simultaneously by $\KL$-divergence and $\TV$ distance, namely:
\begin{theorem}
  \label{cor:zcp-to-tv-kl}
  For any pair $P, Q \in \cM(\Th)$, and any $c \geq 0$, we have
  \begin{align*}
    D_{\ZCP}(P, Q; c)
    &\leq
      2 \sqrt{2 \, D_\TV(P,Q) \, D_{\KL}(P,Q)} + \sqrt{2 \ln(1 + c)} \, D_{\TV}(P_n, P_0)~.
  \end{align*}
\end{theorem}
Note that this control only incurs an additive logarithmic cost in $c$ (recall that $D_{\TV} \leq 1$).
The above is a direct consequence of \Cref{lem:int-to-zcp} and \Cref{prop:Dh-to-tv-kl}, both shown the Appendix (\Cref{sec:int-to-zcp-proof} and \Cref{sec:Dh-to-tv-kl-proof}).
\begin{lemma}\label{lem:int-to-zcp}
  Under conditions of \Cref{thm:zcp-hoeffding}, for any $c \geq 0$,
  \begin{align*}
    D_{\zcp}(P, Q; c)  \le D_\ZCP(P,Q; 1) + 2\sqrt{\ln(2 + 2c)} \, D_{\TV}(P, Q)~.
  \end{align*}
\end{lemma}
\begin{lemma}
  \label{prop:Dh-to-tv-kl}
For any pair $P, Q \in \cM(\Th)$,
  \begin{align*}
    D_\ZCP(P, Q; 1) \leq \sqrt{8D_\TV(P,Q) \, D_{\KL}(P,Q)}~.
  \end{align*}
\end{lemma}
\subsection{Advantage over $\KL$ Divergence in Discrete Cases}
\label{sec:bernoulli-advantage}
We now show that the $\zcp$ divergence can be arbitrarily better than the KL one.
In Section~\ref{sec:main result}, we will use both $D_{\zcp}$ and $\sqrt{D_{\KL}} \cdot D_{\zcp}$ as our new PAC-Bayes bound, and these examples show that in some instances both bounds are superior to the traditional $D_{\KL}$ based bound.

We first consider a basic instance of two Bernoulli random variables with probabilities $p$ and $q$ respectively and then show the case of generic finite support distributions.

For the Bernoulli case, set $q=p/a$ where $a\geq 1$. In particular, we will set $a=\exp(1/p^2)$.
In this way, we have that the total variation is 
\[
1/2(|p-q|+|1-p-(1-q)|)
= |p-q|
= p(1-1/a)~.
\]
On the other hand, the KL divergence is
\[
p \ln \frac{p}{q} + (1-p) \ln \frac{1-p}{1-q}
= p \ln a + (1-p) \ln \frac{1-p}{1-p/a}~.
\]
Observe that the second term is non-positive and we have
\[
(1-p) \ln \frac{1-p}{1-p/a}
= (1-p) \ln (1-p)
- (1-p) \ln (1-p/a)
\geq (1-p) \ln (1-p)
\geq -\exp(-1)~.
\]
Hence, we have that $p \ln a - \exp(-1) \leq D_{\KL} \leq p \ln a$.

We now consider the two cases: 
\begin{itemize} 
\item Setting $a=\exp(1/p^2)$, we have $D_{\KL} \cdot D_{\TV}\leq 1-\frac{1}{a}=1-\exp(-1/p^2)\leq 1- \exp(-1)$ while $1/p-\exp(-1)\leq D_{\KL}\leq 1/p$.
\item Setting $a=\exp(1/p^{3/2})$, we have $D_{\KL} \cdot \sqrt{D_{\TV}}\leq \sqrt{1-\frac{1}{a}}=\sqrt{1-\exp(-1/p^{3/2})}\leq \sqrt{1- \exp(-1)}$ while $1/\sqrt{p}-\exp(-1)\leq D_{\KL}\leq 1/\sqrt{p}$. Note that we could also use this setting for the case above as well.
\end{itemize}

\paragraph{Multivariate instances}
The example above can be easily extended to any pair of distributions with a finite support.
In particular, consider the following generic instance with support of cardinality $d$, and w.l.o.g.\ let $d$ be even. Let $u>0$ be an arbitrary positive number, and set $p= d^{-1-u}$, $a=\exp(d^{\frac{3}{2} u})$. Now define two probability distributions $P=(p_1, \cdots, p_d)$ and $Q=(q_1,\cdots, q_d)$ with weights: 
  \begin{align*}
      p_{i} = \begin{cases}
      p& \text{for $i \in [1:\frac{d}{2}]$}\\
      \fr{1 - \frac{pd}{2}}{d/2}& \text{for $i \in [\frac{d}{2}+1:d]$}
  \end{cases}&, \ \ \   
      q_{i} = \begin{cases}
      \frac{p}{a}& \text{for $i \in [1:\frac{d}{2}]$}\\
      \fr{1 - \frac{pd}{2a}}{d/2}& \text{for $i \in [\frac{d}{2}+1:d]$}
  \end{cases}~.
  \end{align*}  
In this case, for sufficiently large $d$, we have that
\begin{align*}
  D_{\KL}(P,Q) &= \Theta\del{ d\, p\, {\ln(a)} } = \Theta(d^\frac{u}{2})~,\\
  D_{\TV}(P,Q) &= \Theta\del{d\, p} = \Theta(d^{-u})~,\\
  D_\zcp(P,Q) &= \Theta\del{d\, p \, {\sqrt{\ln(a)}}} = \Theta(d^{-\frac{1}{4}u})~.
\end{align*}
We can now check $D_{\zcp}$ based bounds in this instance as follows:
\begin{itemize}
    \item First, $D_\zcp (P,Q)= \Theta(d^{-\frac{1}{4}u})$ is strictly smaller than $D_{\KL}(P,Q)= \Theta(d^{\frac{u}{2}}$). 
    \item Moreover, one could also check that $\sqrt{D_{\KL}} \cdot D_{\zcp} = \Theta(1)$ while $D_{\KL}(P,Q)=\Theta(d^{\frac{u}{2}})$.
\end{itemize}

\subsection{Advantage over $\KL$ Divergence in the Mixture of Gaussian Case}
\label{sec:gaussian-advantage}
Now we consider a continuous case.
In particular, we have the following Gaussian instance, with proof provided in \Cref{sec:gaussian-instance-proof}.
\begin{proposition}
\label{prop:gaussian-instance}
  Let $A=\mathcal{N}(\mu,\sigma_1^2)$ and $B=\mathcal{N}(\mu,\sigma_2^2)$.
  Let  $P$ be a Gaussian mixture $P=p A+(1-p)B$ for $p \in [0,1]$ and let $Q=B$.
  \begin{itemize}
      \item Choosing $\frac{\sigma_2}{\sigma_1} = p$, we have
  $D_{\KL}(P, Q) \geq \frac{1}{2 p} - 1.3$, while  $D_{\TV}(P,Q) \cdot D_{\KL}(P,Q) \le \frac{1}{2}$.
  \item Choosing $\frac{\sigma_2}{\sigma_1}=p^{\frac34}$, we have  $D_{\KL}(P, Q) \geq \frac{1}{2 \sqrt{p}} - 1.22$, while  $ D_{\KL}(P,Q)\cdot \sqrt{D_{\TV}(P,Q) } \le \frac{1}{2}$.
  \end{itemize}
  
\end{proposition}

\section{Main results}\label{sec:main result}
We first present a Hoeffding-type inequality that involves $D_{\zcp}$, which is proved in \Cref{sec:zcp-hoeffding-proof}.
\begin{theorem}[Hoeffding-type $\ZCP$ inequality]
  \label{thm:zcp-hoeffding}
  Let $\delta \in (0, 1)$. Then, for any $P_0 \in \cM_1(\Th)$, with probability at least $1-2 \delta$, simultaneously over $n \in \mathbb{N}, P_n \in \cK(\cX^n, \Th)$, we have\footnote{Simultaneously over $n, P_n$ is understood as $\forall n \in \mathbb{N}~, \forall P \in \cK(\sX^n, \Theta)$ where $\cK()$ is a set of probability kernels as defined in \Cref{sec:def}.}
  \begin{align*}
  \int \Delta(\th) \diff P_n(\th)
  \leq
\frac{\sqrt{2} \, D_{\zcp}\left(P_n, P_0; \frac{\sqrt{2 n}}{\delta}\right) + 2 + \sqrt{\ln\frac{2 \sqrt{n}}{\delta}}}{\sqrt{n}}~.
\end{align*}
\end{theorem}
As shown in \Cref{eq:intro-zcp-bound} the bound is orderwise never worse than the classical $\KL$-based one and in \Cref{sec:divergence} we show instances where thanks to $\ZCP$ divergence it enjoys an improved order.
Moreover, \Cref{thm:zcp-hoeffding} holds \emph{uniformly} over $n \in \mathbb{N}$ unlike most classical bounds which only hold for a fixed $n$.

\begin{remark}
It is possible to obtain a similar inequality by combining the regret guarantee in \citet{zhang2022optimal} and the recently proposed online-to-PAC framework of \citet{lugosi2023online} that obtains PAC bounds from the regret of online learning algorithms. 
Both approaches are valid and we believe both proof methods have their distinct advantages.
In particular, we believe that our proof method is more direct and more flexible. 
Indeed, we show below how to bound the integral of the log optimal wealth, a case that is not covered by the framework in \citet{lugosi2023online} and that allows to recover various known types of inequalities such as the `little-kl' and empirical Bernstein.
\end{remark}
Next, we demonstrate a generalized inequality, which extends beyond the Hoeffding regime of $1/\sqrt{n}$ rate.
In \Cref{sec:betting} we introduced a notion of max-wealth of a betting algorithm.
To state the following result, we parameterize the max-wealth by $\th$:
\begin{align*}
  \Wealth_n^*(\th) 
  \df \max_{\beta \in [-1,1]} \ \prod_{i=1}^n \big(1 + \beta (f(\th, X_i) - \E[f(\th, X_1)]) \big)~.
\end{align*}
The central quantity in the coming result will be the expected \emph{maximal log-wealth} $\int \ln \Wealth^*_n(\th) \diff P_n(\th)$ ---
it was recently shown by \citet{jang2023tighter} that through lower-bounding $\ln \Wealth_n^*$ term we can obtain many known PAC-Bayes bounds.
To this end, our second main result, shown in \Cref{sec:zcp-log-wealth}, gives a bound on the expected maximal log-wealth:

\begin{theorem}[Log-wealth $\ZCP$ inequality]
  \label{thm:zcp-log-wealth-renyi}
  Let $\delta \in (0, 1)$.
  Then, for any $\alpha \in (0,1)$ and for any $P_0 \in \cM_1(\Th)$,
  with probability at least
  $1-2\delta$,
  simultaneously over
  $n \in \mathbb{N} \setminus \{1\}, P_n \in \cK(\cX^n, \Th)$,
\begin{align*}
  &\int \ln \Wealth_n^*(\th) \dP_n(\th)\\  
  &\quad\leq \frac{1}{\sqrt{2}} \sqrt{ \ln\pare{\frac{4 n^2}{\delta}} +  \frac{\alpha}{\alpha - 1} \, \ln(n) +  D_{\alpha}(P_n, P_0)} \, D_{\zcp}\Big(P_n, P_0; \, \frac{\sqrt{2} \, n^{2.5}}{\delta} \Big)\\
    &\qquad + \ln\pare{2 e^2 \sqrt{n} \, \Big(1 + \frac{4 n^2}{\delta}\Big)} + \frac{\delta}{n (n+1)}~.
\end{align*}
\end{theorem}
Note that in addition to $\ZCP$-divergence, now the bound now also depends on the R\'enyi divergence $D_{\alpha}()$ and its order $\alpha$ can be choosen freely.
In particular, in the next corollary we show that asymptotically, when $\alpha$ is tuned based on $n$,
the R\'enyi divergence can be replaced by the $\KL$ divergence.
\begin{corollary}
 \label{prop:thm:zcp-log-wealth-asymptotics}
 Fix $P_0 \in \cM_1(\Th)$. Then, under the conditions of \Cref{thm:zcp-log-wealth-renyi},
with probability one
for all $P \in \cM_1(\Th)$,
 \begin{align*}
   \limsup_{n \to \infty} \frac{\int \ln \Wealth_n^*(\th) \dP(\th)}{\sqrt{2 \ln(n) \ln(e n) \ln(1 + \sqrt{2} \, n^{4.5})}} 
   \leq
   2 + \big(2 + \sqrt{D_{\KL}(P, P_0)}\big) (D_\ZCP(P,P_0;1) + D_{\TV}(P, P_0))~.
\end{align*}
\end{corollary}
A simple consequence of the above, when combined with \Cref{prop:Dh-to-tv-kl}, is that there exists an absolute constant $c > 0$ such that with probability one
\begin{align}
  \label{eq:asymptotic-bound}
  \limsup_{n \to \infty} \ \frac{\int \ln \Wealth^*_n(\th) \dP(\th)}{\ln^{3/2}(n)} \leq c \, \pare{
  1 + \sqrt{D_{\KL}(P, P_0) \, D_{\TV}(P, P_0)} \, \big(1 + \sqrt{D_{\KL}(P, P_0)}\big)
  }~.
\end{align}
In comparison, \citet{jang2023tighter} obtained the bound
\begin{align*}
  \int \ln \Wealth_n^*(\th) \diff P_n(\th) = \cO\pare{D_{\KL}(P_n, P_0) + \ln\frac{n}{\delta}},
\end{align*}
which is looser than the bound in \cref{eq:asymptotic-bound} in terms of dependence on divergence terms, since $D_\KL(P_n,P_0)$ is orderwise at least $D_\KL(P_n,P_0)\sqrt{D_\TV(P_n,P_0)} \ge \sqrt{D_\KL(P_n,P_0)}D_{\zcp}$.

\begin{remark}
Comparing this result to our Hoeffding-style bound (\cref{thm:zcp-hoeffding}) is nontrivial since the left-hand side is written in a different form.
To make a clear comparison, we defer this discussion to \cref{rem:comparison} below, but in summary our optimal log wealth bound leads to a factor of $D_\TV^{1/4}(P_n,P_0)$ looser one compared to \cref{thm:zcp-hoeffding}.
\end{remark}
\begin{remark}
The proof of \Cref{thm:zcp-log-wealth-renyi} is highly non-trivial and there are a few approaches one might think of that fail.
For example, we can successfully upper bound $\int \sqrt{\ln \Wealth_n^*(\th)} \diff P_n(\th)$ but then we cannot obtain an empirical Bernstein inequality because we need a bound like $\int \sqrt{\ln \Wealth_n^*(\th)} \diff P_n(\th)$ $\ge \sqrt{\int \ln \Wealth_n^*(\th)\diff P_n(\th)}$ yet this inequality is the opposite direction of Jensen's inequality. Alternatively, we can upper bound $\int (n\Delta_n(\theta))^2 \diff P_n(\th)$ but then the empirical variance terms will appear integrated with respect to the prior instead of the posterior.
\end{remark}

\subsection{Recovering variants of other known bounds with new divergence}
\label{sec:recovering-bounds}
As we anticipated, inequality in \Cref{thm:zcp-log-wealth-renyi} can be relaxed to obtain various known PAC-Bayes bounds.
First, this was observed by \cite{jang2023tighter}, who derived a bound on $\int \ln \Wealth_n^*(\th) \diff P_n(\th)$ featuring $\KL$-divergence.
By applying similar lower bounding techniques as in their work, in the following we state $\ZCP$ versions of known PAC-Bayes bounds.
In the following let the complexity term be%
\begin{align*}
  \mathrm{Comp}_n(\alpha) = (\text{r.h.s. in \Cref{thm:zcp-log-wealth-renyi}})
  =
  \tilde{\cO}_{\Pr}\pare{
  \sqrt{\frac{\alpha}{\alpha - 1} + D_{\alpha}} \, D_{\zcp}
    +1
  },
\end{align*}
where $\tilde{\cO}_{\Pr}$ is the order in probability notation\footnote{The notation
$
X_n=\cO_{\Pr}\left(a_n\right) \text { as } n \rightarrow \infty
$
means that the set of values $X_n / a_n$ is stochastically bounded. That is, for any $\varepsilon>0$, there exists a finite $M>0$ and a finite $N>0$ such that
$
P\left(|{X_n}/{a_n}|>M\right)<\varepsilon, \forall n>N
$.}
and omits some $\ln(1/\delta)$ terms.

By following the same reasoning as in \Cref{prop:thm:zcp-log-wealth-asymptotics} one can state asymptotic versions of the relaxed bounds, which for a universal constant $c > 0$, will manifest
\begin{align*}
  \fr{\mathrm{Comp}_n\del[2]{1+\fr{1}{\ln(n)} }}{\ln^{3/2}(n)}  \to c \, \left(
  1 + \sqrt{D_{\KL} \, D_{\TV}} \, \left(1 + \sqrt{D_{\KL}}\right)\right)
  \qquad \text{as} \qquad n \to \infty~.
\end{align*}

\paragraph{Empirical Bernstein inequality}
First, we recover an empirical Bernstein-type inequality~\citep{tolstikhin2013pac},
where the bound switches to the `fast rate' $1/n$ when the sample variance is sufficiently small.
In particular, in \Cref{cor:empiricalbernstein} we show
\begin{corollary}
  \label{cor:empiricalbernstein}
  Under the conditions of \Cref{thm:zcp-log-wealth-renyi}, for any $\alpha \in (0,1)$, we have, with probability at least $1-2\dt$, simultaneously over every $n\in\NN\sm\{1\}$ and $P_n \in \cK(\cX^n, \Th)$, 
  \begin{align*}
    \abs{\int \Delta_n(\th) \diff P_n}
    \leq
    \frac{\sqrt{2 \, \mathrm{Comp}_n(\alpha) \, \hat{V}(P_n)}}{(\sqrt{n} - (2/\sqrt{n}) \, \mathrm{Comp}_n(\alpha))_+}
    +
    \frac{2 \mathrm{Comp}_n(\alpha)}{\pare{n - 2 \, \mathrm{Comp}_n(\alpha)}_+},
  \end{align*}
  where $\hat V(P) = \frac{1}{n (n-1)} \sum_{i < j} \int (f(\th, X_i) - f(\th, X_j))^2 \diff P(\th)$ is the expected sample variance.
  Furthermore, there exists an absolute constant $c > 0$ such that with probability one and for all $P \in \cM_1(\Th)$,
  \begin{align*}
  \limsup_{n\to \infty} \ \frac{\abs{\int \Delta_n(\th) \diff P}^2}{\frac1n \pare{\abs{\int \Delta_n(\th) \diff P}  + \hat{V}(P)} \cdot \ln^{3/2}(n)}
  \leq
    c \, \pare{1 + \sqrt{D_{\KL} \, D_{\TV}} \, \big(1 + \sqrt{D_{\KL}}\big)}~.
  \end{align*}
\end{corollary}
We defer the proof to \Cref{sec:empiricalbernstein-proof}.

\begin{remark}\label{rem:comparison}  
  When the sample variance is sufficiently large, that is larger than $\abs{\int \Delta_n(\th) \diff P_n}$,
the bound above provides a comparison point with our own Hoeffding style bound (Theorem~\ref{thm:zcp-hoeffding}) w.r.t.\ the complexity term, which scales with $D_\ZCP(P_n,P_0)$.
  Ignoring the fact that the bound above is asymptotic, we note that the bound in
  Corollary~\ref{cor:empiricalbernstein} scales with
  $\sqrt{D_{\KL} \, D_{\TV}} \cdot \sqrt{D_{\KL}}$, and so
\begin{align*}
    \sqrt{D_{\KL} \, D_{\TV}} \cdot \sqrt{D_{\KL}}
    \geq
    \sqrt{D_{\KL} \, D_{\TV}} \cdot \sqrt{D_{\KL} \, D_{\TV}}
    \geq D_{\ZCP}
  \end{align*}
This means that there is a factor of $D_\TV^{1/2}$ gap.
Investigating whether one can remove this factor is left as future work.
\end{remark}

\paragraph{Bernoulli $\KL$-divergence (Langford-Caruana\,/\,Seeger\,/\,Maurer) bound}
Finally, we consider a tighter inequality for the specific case of Bernoulli $f$ (e.g., Bernoulli losses).
This is a well-studied setting in the PAC-Bayes literature~\citep{LangfordCaruana2001,seeger2002,maurer2004note}.
In particular, 
In this, case we are bounding the $\KL$ divergence between Bernoulli distributions, that is
\begin{align*}
  \kl(p, q) = p \ln\Big(\frac{p}{q}\Big) + (1-p) \ln\Big( \frac{1-p}{1-q} \Big)
\end{align*}
for $p, q \in [0,1]$ such that $p \ll q$;
one can observe that the bound on $\kl()$ is tighter than Hoeffding-type bounds due to Pinsker's inequality.
In such a case, denoting sample average and mean respectively as $\hat{p}_{\th} = (f(\th, X_1) + \dots + f(\th, X_n))/n$ and $p_{\th} = \E[f(\th, X_1)]$, we have
\begin{corollary}
  \label{cor:maurer}
  Under the conditions of \Cref{thm:zcp-log-wealth-renyi}, for any $\alpha \in (0,1)$, we have, with probability at least $1-2\dt$, simultaneously over every $n\in\NN\sm\{1\}$ and $P_n \in \cK(\cX^n, \Th)$, 
  \begin{align*}
  \kl\left(\int \hat{p}_{\th} \diff P_n(\theta), \int p_{\th} \diff P_n(\theta) \right)
  \leq
  \frac{\mathrm{Comp}_n(\alpha)}{n}~.
\end{align*}
Furthermore, there exists an absolute constant $c > 0$ such that with probability one and for all $P \in \cM_1(\Th)$,
  \begin{align*}
    \limsup_{n\to \infty} \ \frac{\kl\left(\int \hat{p}_{\th} \diff P(\theta), \int p_{\th} \diff P(\theta) \right)}{\frac1n \cdot \ln^{3/2}(n)}
  \leq
    c \, \pare{1 + \sqrt{D_{\KL} \, D_{\TV}} \, \big(1 + \sqrt{D_{\KL}}\big) }~.
  \end{align*}
\end{corollary}
The proof of \Cref{cor:maurer} closely follows that of \citet[Proposition 3 (Sec. A.2)]{jang2023tighter}.
\section{Conclusions}
In this paper we derived, to the best of our knowledge, a first high-probability PAC-Bayes bound with the $\ZCP$ divergence. This divergence is never worse than the \ac{KL} divergence orderwise and it enjoys a strictly better scaling in some instances.
In the concentration regime $1/\sqrt{n}$ for the deviation $\int \fr1n \sum_{t=1}^n \Delta_t(\th)\diff P_n(\th)$, the new bound scales with $D_{\ZCP} = \tilde{\cO}(\sqrt{D_{\KL} \, D_{\TV}})$ instead of $\cO(\sqrt{D_{\KL}})$.
In other regimes, such as the Bernstein regime, the bound asymptotically scales with $\sqrt{\sqrt{D_{\KL}} \, D_{\zcp}}$, which can be analyzed to be a factor of $D_\TV^{1/4}$
worse than our Hoeffding-style bound.
Both proofs rely on a novel change-of-measure argument with respect to $x \mapsto e^{x^2/2}$  potential which might be of an independent interest.

A tantalizing open problem is whether our bounds can be further improved.
It would be interesting to see if it is possible to establish some (Pareto) optimalities for PAC-Bayes bounds.

\section{Proof of \Cref{thm:zcp-hoeffding}: McAllester/Hoeffding-type bound}
\label{sec:zcp-hoeffding-proof}
First, we need the following lemmas.
\begin{lemma}[{\citet[Lemma~9.7]{Orabona19}}]
  \label{lemma:lamb}
  Let function $F^{\star}(x) = b \, e^{x^2/(2a)}$ for $a, b > 0$. Then, we have
  $
    F(y) \leq |y| \sqrt{a \ln\pare{1 + \frac{a y^2}{b^2}}} - b
  $.
\end{lemma}

\begin{theorem}[{\citet{Cesa-BianchiL06}}]
\label{thm:regret_up}
For any sequence of $c_i \in [-1,1]$, there exists an online algorithm that selects in $\beta_i \in [-1,1]$ with knowledge of $c_1, \dots, c_{t-1}$ such that for all $n$ it uniformly guarantees
\[
\prod_{i=1}^n (1+\beta_i c_i)
\geq \frac{1}{\sqrt{2(n+1)}} \max_{\beta \in [-1, 1]} \ \prod_{i=1}^n (1+\beta c_i)~.
\]
\end{theorem}
The following lemma is shown in the appendix.
\begin{lemma}
\label{lemma:wealth_lower_bound}
For any $c_1, \dots, c_n \in [-1,1]$, we have
\[
\max_{\beta \in [-1, 1]} \ \prod_{i=1}^n (1+\beta c_i) 
\geq \exp\left(\frac{\left(\sum_{i=1}^n c_i\right)^2}{4n}\right)~.
\]
\end{lemma}
Now we proceed with the proof of \Cref{thm:zcp-hoeffding}.
Throughout the proof it will be convenient to work with unnormalized gap (instead of normalized one, $\Delta$):
\begin{align*}
  \bar{\Delta} = \sum_{i=1}^n \bar{\Delta}_i~, \qquad \bar{\Delta}_i = f(\th, X_i) - \E[f(\th, X_1)]~.
\end{align*}
  Consider a change-of-measure decomposition
\begin{align*}
  \int \bar{\Delta}(\th) \diff P_n(\th)
  \leq
  \underbrace{
  \int \bar{\Delta}(\th)  \cd \del{\fr{\diff P_n}{\diff P_0}(\th) - 1} \diff P_0(\th)
  }_{(i)}
  +
  \underbrace{\int \bar{\Delta}(\th) \diff P_0(\th)}_{(ii)}
\end{align*}
and note right away that by the fact that term $(ii)$ can be controlled by a time-uniform Hoeffding-style concentration inequality since $\int (f(\th, X_i) - \E[f(\th, X_1)])  \diff P_0 \in [-1,1]$ (see \Cref{sec:betting}).
Namely, with probability at least $1-\delta$, simultaneously for all $n \in \mathbb{N}$,
\begin{align*}
 \int \bar{\Delta}(\th) \diff P_0(\th) \leq \sqrt{n \ln\frac{2 \sqrt{n}}{\delta}}~.
\end{align*}
Now we turn our attention to term $(i)$.
By Fenchel-Young inequality, for a convex $F : \R \to \R$,
\begin{align}
  \label{eq:FY-decomp-hoef}
  (i) \leq
  \int F^{\star}(\bar{\Delta}(\theta)) \diff P_0(\th)  + \int F\del{\fr{\diff P_n}{\diff P_0}(\th) - 1} \diff P_0(\th)~.
\end{align}
Now, we will make a particular choice of $F()$ by using \Cref{lemma:lamb} choosing $a = 2 n$ and leaving $b$ to be tuned later, $
  F^{\star}(\bar{\Delta}(\th))
  =
    b \, \exp\pare{\frac{\bar{\Delta}^2(\th)}{4 n}}$.
Throughout the rest of the proof we will control the above.
In particular, we make a key observation that the above term is controlled by the \emph{maximal wealth} achievable by some online algorithm and using its regret bound we can argue that $\int F^{\star}(\bar{\Delta}^2(\th)) \diff P_0(\th)$ concentrates well.
In particular, we will need \Cref{lemma:wealth_lower_bound} (shown in the appendix),
which is then combined with the regret bound of \Cref{thm:regret_up}.
Suppose that $B_{i-1}(\th)$ is a prediction of a betting algorithm after observing $\bar{\Delta}_1(\th), \ldots \bar{\Delta}_{i-1}(\th)$, and let its wealth at step $n$ be defined as
\begin{align}
  \label{eq:W-theta}
  \Wealth_n(\th) = \prod_{i=1}^n (1 + B_{i-1}(\th) \bar{\Delta}_i(\th))~.
\end{align}
Then, wealth is related to the max-wealth through \Cref{thm:regret_up}, $\Wealth^*(\th) \leq \sqrt{2(n+1)} \, \Wealth_n(\th)$.
The final bit to note that $\int W_n(\th) \diff P_0(\th)$ is a martingale, that is\footnote{We use notation $\E_{n-1}[\cdot] = \E[\cdot \mid X_1, \ldots, X_{n-1}]$.}
\begin{align}
  \label{eq:int-W-is-martingale}
  \E_{n-1} \int W_n(\th) \diff P_0(\th)
  = \int \E_{n-1} W_n(\th) \diff P_0(\th) 
  = \int W_{n-1}(\th) \diff P_0(\th),
\end{align}
where we used Fubini's theorem to exchange $\E$ and $\int$.
This fact allows us to use Ville's inequality (\Cref{thm:ville}).
So, we obtain
\begin{align*}
  \int F^{\star}(\bar{\Delta}(\th)) \diff P_0(\th)
  &\leq
    b \, \int \exp\pare{\sum_{i=1}^n \ln(1 + B_{i-1} \bar{\Delta}_i(\th)) + \ln \sqrt{2(n+1)}} \diff P_0(\th)\\
  &=
    b \, \sqrt{2(n+1)} \, \int W_n(\th) \diff P_0(\th)
\leq
    b \, \frac{\sqrt{2(n+1)}}{\delta}  \tag{By Ville's inequality}\\
  &=
    \sqrt{2(n+1)} \leq 2 \sqrt{n}~. \tag{Tuning $b = \delta$}
\end{align*}
That said, using Lemma \ref{lemma:lamb} and the above provide a bound on \cref{eq:FY-decomp-hoef} that is
\begin{align*}
  (i)
  &\leq \sqrt{n} +
    \sqrt{2 n} \, \int \abs{\fr{\diff P_n}{\diff P_0}(\th) - 1} \sqrt{\ln\pare{1 + \frac{2 n}{\delta^2} \pare{\fr{\diff P_n}{\diff P_0}(\th) - 1}^2 }} \diff P_0(\th) - \delta\\
  &=
    \sqrt{n} + \sqrt{2 n} \, D_{\zcp}\left(P_n, P_0; \frac{\sqrt{2n}}{\delta}\right) - \delta~.
\end{align*}
Completing the bound and dividing it though by $n$ completes the proof.
\jmlrQED 
\section*{Acknowledgements}
Kwang-Sung Jun was supported in part by National Science Foundation under grant CCF-2327013.

\bibliography{learning}

\begin{thebibliography}{47}
\providecommand{\natexlab}[1]{#1}
\providecommand{\url}[1]{\texttt{#1}}
\expandafter\ifx\csname urlstyle\endcsname\relax
  \providecommand{\doi}[1]{doi: #1}\else
  \providecommand{\doi}{doi: \begingroup \urlstyle{rm}\Url}\fi

\bibitem[Alquier(2024)]{alquier2024user}
P.~Alquier.
\newblock User-friendly introduction to {PAC-B}ayes bounds.
\newblock \emph{Foundations and Trends{\textregistered} in Machine Learning},
  17\penalty0 (2):\penalty0 174--303, 2024.

\bibitem[Alquier and Guedj(2018)]{alquier2018simpler}
P.~Alquier and B.~Guedj.
\newblock Simpler {PAC-Bayesian} bounds for hostile data.
\newblock \emph{Machine Learning}, 107\penalty0 (5):\penalty0 887--902, 2018.

\bibitem[Awasthi et~al.(2020)Awasthi, Kale, Karp, and Mohri]{awasthi2020pac}
P.~Awasthi, S.~Kale, S.~Karp, and M.~Mohri.
\newblock {PAC-Bayes} learning bounds for sample-dependent priors.
\newblock In \emph{Advances in Neural Information Processing Systems},
  volume~33, pages 4403--4414, 2020.

\bibitem[B{\'e}gin et~al.(2016)B{\'e}gin, Germain, Laviolette, and
  Roy]{begin2016pac}
L.~B{\'e}gin, P.~Germain, F.~Laviolette, and J.-F. Roy.
\newblock {PAC-B}ayesian bounds based on the r{\'e}nyi divergence.
\newblock In \emph{International Conference on Artificial Intelligence and
  Statistics (AISTATS)}, pages 435--444. PMLR, 2016.

\bibitem[Catoni(2007)]{catoni2007}
O.~Catoni.
\newblock {PAC-B}ayesian supervised classification: The thermodynamics of
  statistical learning.
\newblock IMS Lecture Notes-Monograph Series, 56, 2007.

\bibitem[Cesa-Bianchi and Lugosi(2006)]{Cesa-BianchiL06}
N.~Cesa-Bianchi and G.~Lugosi.
\newblock \emph{Prediction, learning, and games}.
\newblock Cambridge University Press, 2006.

\bibitem[Cesa-Bianchi and Orabona(2021)]{cesa2021online}
N.~Cesa-Bianchi and F.~Orabona.
\newblock Online learning algorithms.
\newblock \emph{Annual review of statistics and its application}, 8:\penalty0
  165--190, 2021.

\bibitem[Chaudhuri et~al.(2009)Chaudhuri, Freund, and
  Hsu]{chaudhuri2009parameter}
K.~Chaudhuri, Y.~Freund, and D.~J. Hsu.
\newblock A parameter-free hedging algorithm.
\newblock \emph{Advances in Neural Information Processing Systems}, 22, 2009.

\bibitem[Chu and Raginsky(2023)]{chu2023unified}
Y.~Chu and M.~Raginsky.
\newblock A unified framework for information-theoretic generalization bounds.
\newblock In \emph{Advances in Neural Information Processing Systems}, 2023.

\bibitem[Cover(1974)]{Cover74}
T.~M Cover.
\newblock Universal gambling schemes and the complexity measures of
  {Kolmogorov} and {Chaitin}.
\newblock Technical Report~12, Department of Statistics, Stanford University,
  1974.
\newblock URL \url{https://purl.stanford.edu/js411qm9805}.

\bibitem[Cover(1991)]{Cover91}
T.~M. Cover.
\newblock Universal portfolios.
\newblock \emph{Mathematical Finance}, pages 1--29, 1991.

\bibitem[Donsker and Varadhan(1975)]{DoVa75}
M.~D. Donsker and S.~S. Varadhan.
\newblock Asymptotic evaluation of certain {M}arkov process expectations for
  large time.
\newblock \emph{Communications on Pure and Applied Mathematics}, 28, 1975.

\bibitem[Dziugaite and Roy(2017)]{dziugaite2017computing}
G.~K. Dziugaite and D.~M. Roy.
\newblock {Computing Nonvacuous Generalization Bounds for Deep (Stochastic)
  Neural Networks with Many More Parameters than Training Data}.
\newblock In \emph{Uncertainty in Artificial Intelligence (UAI)}, 2017.

\bibitem[Dziugaite et~al.(2021)Dziugaite, Hsu, Gharbieh, Arpino, and
  Roy]{dziugaite2021role}
G.~K. Dziugaite, K.~Hsu, W.~Gharbieh, G.~Arpino, and D.~M. Roy.
\newblock On the role of data in {PAC-B}ayes bounds.
\newblock In \emph{International Conference on Artificial Intelligence and
  Statistics (AISTATS)}, 2021.

\bibitem[Fan et~al.(2015)Fan, Grama, and Liu]{fan2015exponential}
X.~Fan, I.~Grama, and Q.~Liu.
\newblock Exponential inequalities for martingales with applications.
\newblock \emph{Electronic Journal of Probability}, 20:\penalty0 1--22, 2015.

\bibitem[Haddouche and Guedj(2023)]{haddouche2022pac}
M.~Haddouche and B.~Guedj.
\newblock {PAC}-{Bayes} with unbounded losses through supermartingales.
\newblock \emph{Transactions on Machine Learning Research (TMLR)}, 2023.

\bibitem[Jang et~al.(2023)Jang, Jun, Kuzborskij, and Orabona]{jang2023tighter}
K.~Jang, K.-S. Jun, I.~Kuzborskij, and F.~Orabona.
\newblock Tighter {PAC-Bayes} bounds through coin-betting.
\newblock In \emph{Conference on Computational Learning Theory (COLT)}, 2023.

\bibitem[Jun and Orabona(2019)]{JunO19}
K.-S. Jun and F.~Orabona.
\newblock Parameter-free online convex optimization with sub-exponential noise.
\newblock In \emph{Proc. of the Conference on Learning Theory (COLT)}, 2019.

\bibitem[Kallenberg(2017)]{kallenberg2017}
O.~Kallenberg.
\newblock \emph{Random Measures, Theory and Applications}.
\newblock Springer, 2017.

\bibitem[Kelly(1956)]{Kelly56}
J.~L. Kelly, jr.
\newblock A new interpretation of information rate.
\newblock \emph{IRE Transactions on Information Theory}, 2\penalty0
  (3):\penalty0 185--189, 1956.

\bibitem[Koolen and Van~Erven(2015)]{koolen2015second}
W.~M. Koolen and T.~Van~Erven.
\newblock Second-order quantile methods for experts and combinatorial games.
\newblock In \emph{Conference on Computational Learning Theory (COLT)}, pages
  1155--1175. PMLR, 2015.

\bibitem[Kuzborskij and Szepesvári(2019)]{kuzborskij2019efron}
I.~Kuzborskij and Cs. Szepesvári.
\newblock Efron-{S}tein {PAC}-{B}ayesian {I}nequalities.
\newblock arXiv:1909.01931, 2019.

\bibitem[Langford and Caruana(2001)]{LangfordCaruana2001}
J.~Langford and R.~Caruana.
\newblock {(Not) bounding the true error}.
\newblock In \emph{Advances in Neural Information Processing Systems}, pages
  809--816, 2001.

\bibitem[Lugosi and Neu(2023)]{lugosi2023online}
G.~Lugosi and G.~Neu.
\newblock Online-to-{PAC} conversions: Generalization bounds via regret
  analysis.
\newblock arXiv:2305.19674, 2023.

\bibitem[Luo and Schapire(2015)]{LuoS15}
H.~Luo and R.~E. Schapire.
\newblock Achieving all with no parameters: {AdaNormalHedge}.
\newblock In \emph{Proc. of the Conference on Learning Theory (COLT)}, pages
  1286--1304, 2015.

\bibitem[Maurer(2004)]{maurer2004note}
A.~Maurer.
\newblock A note on the {PAC} {B}ayesian theorem.
\newblock \emph{arXiv:0411099}, 2004.

\bibitem[McAllester(1998)]{McAllester98}
D.~McAllester.
\newblock Some {PAC}-{B}ayesian theorems.
\newblock In \emph{Proceedings of the eleventh annual conference on
  Computational learning theory}, pages 230--234, 1998.

\bibitem[Mhammedi et~al.(2019)Mhammedi, Gr\"{u}nwald, and Guedj]{MhammediGG19}
Z.~Mhammedi, P.~Gr\"{u}nwald, and B.~Guedj.
\newblock {PAC}-{B}ayes un-expected {B}ernstein inequality.
\newblock In H.~Wallach, H.~Larochelle, A.~Beygelzimer, F.~d\textquotesingle
  Alch\'{e}-Buc, E.~Fox, and R.~Garnett, editors, \emph{Advances in Neural
  Information Processing Systems}, volume~32. Curran Associates, Inc., 2019.

\bibitem[Nguyen et~al.(2010)Nguyen, Wainwright, and Jordan]{nguyen10estimating}
X.~Nguyen, M.~J. Wainwright, and M.~I. Jordan.
\newblock Estimating divergence functionals and the likelihood ratio by convex
  risk minimization.
\newblock \emph{IEEE Transactions on Information Theory}, 56\penalty0
  (11):\penalty0 5847--5861, 2010.

\bibitem[Nielsen and Sun(2018)]{nielsen2018guaranteed}
F.~Nielsen and K.~Sun.
\newblock Guaranteed deterministic bounds on the total variation distance
  between univariate mixtures.
\newblock In \emph{2018 IEEE 28th International Workshop on Machine Learning
  for Signal Processing (MLSP)}, pages 1--6. IEEE, 2018.

\bibitem[Ohnishi and Honorio(2021)]{ohnishi2021novel}
Y.~Ohnishi and J.~Honorio.
\newblock Novel change of measure inequalities with applications to
  {PAC-B}ayesian bounds and {M}onte {C}arlo estimation.
\newblock In \emph{International Conference on Artificial Intelligence and
  Statistics (AISTATS)}, pages 1711--1719. PMLR, 2021.

\bibitem[Orabona(2019)]{Orabona19}
F.~Orabona.
\newblock A modern introduction to online learning.
\newblock \emph{arXiv preprint arXiv:1912.13213}, 2019.
\newblock URL \url{https://arxiv.org/abs/1912.13213}.
\newblock Version 6.

\bibitem[Orabona and Jun(2024)]{orabona2023tight}
F.~Orabona and K.-S. Jun.
\newblock Tight concentrations and confidence sequences from the regret of
  universal portfolio.
\newblock \emph{IEEE Transactions on Information Theory}, 70\penalty0
  (1):\penalty0 436--455, 2024.

\bibitem[Orabona and P\'al(2016)]{OrabonaP16}
F.~Orabona and D.~P\'al.
\newblock Coin betting and parameter-free online learning.
\newblock In D.~D. Lee, M.~Sugiyama, U.~V. Luxburg, I.~Guyon, and R.~Garnett,
  editors, \emph{Advances in Neural Information Processing Systems 29}, pages
  577--585. Curran Associates, Inc., 2016.

\bibitem[P{\'e}rez-Ortiz et~al.(2021)P{\'e}rez-Ortiz, Rivasplata, Shawe-Taylor,
  and Szepesv{\'a}ri]{perez-ortiz2020tighter}
M.~P{\'e}rez-Ortiz, O.~Rivasplata, J.~Shawe-Taylor, and Cs. Szepesv{\'a}ri.
\newblock {Tighter risk certificates for neural networks}.
\newblock \emph{Journal of Machine Learning Research}, 2021.

\bibitem[Rakhlin and Sridharan(2017)]{RakhlinS17}
A.~Rakhlin and K.~Sridharan.
\newblock On equivalence of martingale tail bounds and deterministic regret
  inequalities.
\newblock In \emph{Proc. of the Conference On Learning Theory (COLT)}, pages
  1704--1722, 2017.

\bibitem[Rivasplata et~al.(2020)Rivasplata, Kuzborskij, Szepesv{\'a}ri, and
  Shawe-Taylor]{rivasplata2020pac}
O.~Rivasplata, I.~Kuzborskij, Cs. Szepesv{\'a}ri, and J.~Shawe-Taylor.
\newblock {PAC-B}ayes analysis beyond the usual bounds.
\newblock In \emph{Advances in Neural Information Processing Systems},
  volume~33, pages 16833--16845, 2020.

\bibitem[Ruderman et~al.(2012)Ruderman, Reid, Garc{\'\i}a-Garc{\'\i}a, and
  Petterson]{ruderman2012tighter}
A.~Ruderman, M.~Reid, D.~Garc{\'\i}a-Garc{\'\i}a, and J.~Petterson.
\newblock Tighter variational representations of $f$-divergences via
  restriction to probability measures.
\newblock \emph{International Conference on Machine Learing (ICML)}, 2012.

\bibitem[Seeger(2002)]{seeger2002}
M.~Seeger.
\newblock {PAC-B}ayesian generalisation error bounds for gaussian process
  classification.
\newblock \emph{Journal of Machine Learning Research}, 3\penalty0
  (Oct):\penalty0 233--269, 2002.

\bibitem[Seldin et~al.(2012)Seldin, Laviolette, Cesa-Bianchi, Shawe-Taylor, and
  Auer]{seldin2012pac}
Y.~Seldin, F.~Laviolette, N.~Cesa-Bianchi, J.~Shawe-Taylor, and P.~Auer.
\newblock {PAC-B}ayesian inequalities for martingales.
\newblock \emph{IEEE Transactions on Information Theory}, 58\penalty0
  (12):\penalty0 7086--7093, 2012.

\bibitem[Shafer and Vovk(2001)]{ShaferV05}
G.~Shafer and V.~Vovk.
\newblock \emph{Probability and finance: it's only a game!}
\newblock John Wiley \& Sons, 2001.

\bibitem[Shawe-Taylor and Williamson(1997)]{shawetaylor1997pac}
J.~Shawe-Taylor and R.~C. Williamson.
\newblock A {PAC} analysis of a {B}ayesian estimator.
\newblock In \emph{Conference on Computational Learning Theory (COLT)}, 1997.

\bibitem[Tolstikhin and Seldin(2013)]{tolstikhin2013pac}
I.~O. Tolstikhin and Y.~Seldin.
\newblock {PAC}-{B}ayes-empirical-{B}ernstein inequality.
\newblock In C.J. Burges, L.~Bottou, M.~Welling, Z.~Ghahramani, and K.Q.
  Weinberger, editors, \emph{Advances in Neural Information Processing
  Systems}, volume~26. Curran Associates, Inc., 2013.

\bibitem[Van~Erven and Harremos(2014)]{van2014renyi}
T.~Van~Erven and P.~Harremos.
\newblock R{\'e}nyi divergence and {K}ullback-{L}eibler divergence.
\newblock \emph{IEEE Transactions on Information Theory}, 60\penalty0
  (7):\penalty0 3797--3820, 2014.

\bibitem[Ville(1939)]{Ville39}
J.~Ville.
\newblock \emph{Étude critique de la notion de collectif}.
\newblock Gauthier-Villars, Paris, 1939.
\newblock URL \url{http://archive.numdam.org/item/THESE_1939__218__1_0/}.

\bibitem[Yang et~al.(2019)Yang, Sun, and Roy]{yang2019fast}
J.~Yang, S.~Sun, and D.~M. Roy.
\newblock Fast-rate {PAC-B}ayes generalization bounds via shifted {R}ademacher
  processes.
\newblock \emph{Advances in Neural Information Processing Systems}, 32, 2019.

\bibitem[Zhang et~al.(2022)Zhang, Cutkosky, and Paschalidis]{zhang2022optimal}
Z.~Zhang, A.~Cutkosky, and Y.~Paschalidis.
\newblock Optimal comparator adaptive online learning with switching cost.
\newblock \emph{Advances in Neural Information Processing Systems},
  35:\penalty0 23936--23950, 2022.

\end{thebibliography}

\newpage

\appendix

\setlength{\abovedisplayskip}{3pt}
\setlength{\belowdisplayskip}{3pt}
\setlength{\abovedisplayshortskip}{3pt}
\setlength{\belowdisplayshortskip}{3pt}

\section{Proof of \Cref{thm:zcp-log-wealth-renyi}: Bound on log-wealth}
\label{sec:zcp-log-wealth}
Recall that $\Wealth_n(\th)$ is wealth of a betting algorithm as defined in \cref{eq:W-theta}.
In the following we will work with its truncated counterpart, that is
$\Wealth^1_n(\th) = 1 \vee \Wealth_n(\th)$ which satisfies $\Wealth^1_n(\th) \le 1 + \Wealth_n(\th)$.

The proof is similar to that of \Cref{thm:zcp-hoeffding}, however 
it has a different structure than the usual PAC-Bayes bounds and we believe it might be interesting on its own.
Here, we are working with $\ln \Wealth^1_n(\th)$ instead of $\Delta(\th)$.
As in the proof of \Cref{thm:zcp-hoeffding} we apply Fenchel-Young, but with a family of functions, $F_{\theta}^*(x) = \exp(\frac{x^2}{2 a})$, and we will choose a different $a$ for each $\theta$, that is, $a = \ln \Wealth^1_n(\th) / 2$, which will lead to the term $\ln \Wealth^1_n(\th)$ as a part of the divergence.
Then, the main idea of this proof is to control the $\ln \Wealth^1_n(\th)$ term in the divergence using the fact that $\ln \Wealth^1_n(\th)$ is often small. Therefore we condition on the event when it is small and the remaining part we control pessimistically, i.e., $\ln \Wealth^2 \leq n$.
\Cref{prop:P-lnW-t} tells us when $\ln \Wealth$ is small.
\begin{lemma}
  \label{prop:P-lnW-t}
  Let $Q \in \cM_1(\Th)$ be independent from data.
  Then, for any $\delta \in (0,1)$ and any $t > 0$,
  \begin{align*}
    \Pr\left\{\Pr_Q\pare{\Wealth_n^1(\th) \ge t} < \frac{2}{t \delta}\right\} \ge 1 - \delta~.
  \end{align*}
\end{lemma}
\begin{proof}
  For every $\th \in \Th$, Markov's inequality implies that
  \begin{align*}
    \PP_Q\{\Wealth_n^1(\th) \ge t\}
    \le \fr{\EE_Q \Wealth_n^1(\th)}{t}~.
  \end{align*}
  Furthermore, another application of Markov's inequality (with respect to data) implies that
  \begin{align*}
    \PP\left\{\EE_Q \Wealth_n^1(\th) \ge \fr{2}{\dt}\right\}
    \le \fr{\dt}{2}  \EE \EE_Q \Wealth_n^1(\th)
    \le \fr{\dt}{2} \EE \EE_Q [1 + \Wealth_n(\th)]
    =   \dt,
  \end{align*}
  where we have used the fact that $\EE \EE_Q \Wealth_n(\th) = \EE_Q \EE \Wealth_n(\th)$ by Fubini's theorem and $\EE \Wealth_n(\th) = 1$.
  We conclude the proof by chaining the two displays above:
  \begin{align*}
    1-\dt 
    \le \PP\left\{\EE_Q \Wealth_n^1(\th) < \fr{2}{\dt}\right\} 
    \le \PP\left\{t\PP_Q(\Wealth_n^1(\th) \ge t) < \fr2\dt\right\}~.
  \end{align*}
  \vfill
\end{proof}
Consider the following decomposition w.r.t.\ a free parameter $t > 0$ to be tuned later:
\begin{align*}
  &\int \ln \Wealth_n(\th) \diff P_n(\th)\\
  &\quad\leq
  \int \ln \Wealth^1_n(\th) \diff P_n(\th)\\
  &\quad=
    \int \ln (\Wealth^1_n(\th))  \cd \del{\fr{\diff P_n}{\diff P_0}(\th)} \, \diff P_0(\th)\\
  &\quad= \int \ln (\Wealth^1_n(\th))  \cd \del{\fr{\diff P_n}{\diff P_0}(\th)} \, \onec{\Wealth^1_n(\th) > t} \diff P_0(\th)\\
       &\qquad+ \int \ln (\Wealth^1_n(\th))  \cd \del{\fr{\diff P_n}{\diff P_0}(\th)} \, \onec{\Wealth^1_n(\th) \leq t} \diff P_0(\th)\\
     &\quad \leq \underbrace{
       n \int  \del{\fr{\diff P_n}{\diff P_0}(\th)} \, \onec{\Wealth^1_n(\th) > t} \diff P_0(\th)
  }_{(i)}\\
     &\qquad+ \underbrace{
       \int \ln (\Wealth^1_n(\th))  \cd \del{\fr{\diff P_n}{\diff P_0}(\th) - 1} \, \onec{\Wealth^1_n(\th) \leq t} \diff P_0(\th)
       }_{(ii)}\\
  &\qquad+ \underbrace{
       \int \ln (\Wealth^1_n(\th)) \onec{\Wealth^1_n(\th) \leq t} \diff P_0(\th)
       }_{(iii)},
\end{align*}
where to get $(i)$ we have upper bounded $\ln \Wealth^1_n(\theta)$ with the pessimistic upper bound $\ln(1+ 2^n) \leq n$.~%
\paragraph{Bounding $(iii)$.}
We note right away that by the fact that $\int \Wealth_n(\th) \diff P_0$ is a martingale (see \cref{eq:int-W-is-martingale}), we can use Ville's inequality
(\Cref{thm:ville}) to have
\begin{align*}
  (iii)
  \le \int \ln \Wealth^1_n(\th) \diff P_0(\th) 
  &\le \ln \left(\int \Wealth^1_n(\th) \diff P_0(\th)\right) \tag{Jensen's inequality}
\\&\le \ln \left(1 + \int \Wealth_n(\th) \diff P_0(\th)\right) 
\\&\leq \ln\left(1 + \fr1\dt\right)~. \tag{Ville's inequality; w.p. $\ge 1 - \dt$ }
\end{align*}
Now we handle remaining terms $(i)$ and $(ii)$.
\paragraph{Bounding $(i)$.}
For the term $(i)$, we have
\begin{align*}
(i) &= n \, \int  \del{\fr{\diff P_n}{\diff P_0}(\th)} \, \onec{\Wealth^1_n(\th) > t} \diff P_0(\th)\\
&\stackrel{(a)}{\leq} n \, \pare{\Pr_{P_0}\pare{\Wealth^1_n(\th) > t}}^{1 - \frac{1}{\alpha}}
\underbrace{\pare{\int \abs{\fr{\diff P_n}{\diff P_0}(\th)}^{\alpha} \diff P_0(\th)}^{\frac{1}{\alpha}}}_{I_{\alpha}} \tag{$\alpha > 1$}\\
&\leq \frac{n I_{\alpha}}{\pare{t \delta / 2}^{1 - \frac{1}{\alpha}}} \tag{By \Cref{prop:P-lnW-t}; w.p. $\ge 1-\dt$ },
\end{align*}
where step $(a)$ comes by H\"older's inequality.
\paragraph{Bounding $(ii)$.}
By Fenchel-Young inequality, for a family of convex $F_\theta : \R \to \R$,
\begin{align*}
(ii) &= \int \ln (\Wealth^1_n(\th))  \cd \del{\fr{\diff P_n}{\diff P_0}(\th) - 1} \, \onec{\Wealth^1_n(\th) \leq t} \diff P_0(\th) \\
&\leq \int F^{\star}_{\theta}(\ln (\Wealth^1_n(\th))) \diff P_0(\th)  + \int F_\theta\del{\fr{\diff P_n}{\diff P_0}(\th) - 1} \, \onec{\Wealth^1_n(\th) \leq t} \diff P_0(\th)~.
\end{align*}
We use Lemma~\ref{lemma:lamb}, choosing $a = \ln(\Wealth^1_n(\th)) / 2$ and $b = \delta$ to have
\[
  F^{\star}_{\theta}(\ln(\Wealth^1_n(\th)))
  =
  \delta \, \exp\pare{\frac{\ln^2(\Wealth^1_n(\th))}{\ln(\Wealth^1_n(\th))}}
  =
    \delta \, \Wealth^1_n(\th)
\]
and so
\begin{align*}
  \int F^{\star}_{\theta}(\ln (\Wealth^1_n(\th))) \diff P_0(\th)
  =
  \delta \, \int \Wealth^1_n(\th) \diff P_0(\th)
  \leq
  \delta
  +
  \delta \int \Wealth_n(\th) \diff P_0(\th)
\leq 1+\delta,
\end{align*}
where the last inequality is a consequence of Ville's inequality.
Finally, we complete the bound on $(ii)$ by having a chain of inequalities on
\begin{align*}
\int &F\del{\fr{\diff P_n}{\diff P_0}(\th) - 1} \, \onec{\Wealth^1_n(\th) \leq t} \diff P_0(\th) \\
&\stackrel{(b)}{\leq} \int \abs{\fr{\diff P_n}{\diff P_0}(\th) - 1} \sqrt{\frac12 \ln (\Wealth^1_n(\th)) \ln\pare{1 + \frac{\ln (\Wealth^1_n(\th))}{2 \delta^2} \, \del{\fr{\diff P_n}{\diff P_0}(\th) - 1}^2}} \onec{\Wealth^1_n(\th) \leq t}\diff P_0(\th)\\
     &\leq \sqrt{\frac{\ln(t)}{2}}
       \int \abs{\fr{\diff P_n}{\diff P_0}(\th) - 1} \sqrt{\ln\pare{1 + \frac{\ln (\Wealth^1_n(\th))}{2 \delta^2} \, \del{\fr{\diff P_n}{\diff P_0}(\th) - 1}^2} } \diff P_0(\th)\\
       &\leq \sqrt{\frac{\ln(t)}{2}} \, D_{\zcp}\Big(P_n, P_0; \, \sqrt{\frac{n}{2\delta^2}} \Big),
\end{align*}
where $(b)$ comes by \Cref{lemma:lamb} and where we once more used a bound $\ln \Wealth^1_n(\th) \leq n$.
\paragraph{Tuning $t$ and completing the proof.}
Putting all together, we have
\begin{align*}
\int \ln \Wealth_n(\th) \diff P_n(\th)
  &\leq \sqrt{\frac{\ln(t)}{2}} \, D_{\zcp}\Big(P_n, P_0; \, \sqrt{\frac{n}{2\delta^2}} \Big)  + \frac{n I_{\alpha}}{\pare{t \delta / 2}^{1 - \frac{1}{\alpha}}}+ \ln\left(1 + \frac{1}{\delta}\right) + 1 + \delta~.
\end{align*}
and setting $t = \frac{2}{\delta} \, (n I_{\alpha})^{\frac{1}{1 - \frac{1}{\alpha}}}$ we obtain
\begin{align*}
  &\int \ln \Wealth_n(\th) \diff P_n(\th)\\
  &\quad\leq \frac{1}{\sqrt{2}}\sqrt{\ln\pare{\frac{2}{\delta}} +  \frac{\alpha}{\alpha - 1} \, \ln\pare{n I_{\alpha}}} \, D_{\zcp}\Big(P_n, P_0; \, \sqrt{\frac{n}{2\delta^2}} \Big) + \ln\left(1 + \frac{1}{\delta}\right) + 2 + \delta\\
  &\quad\leq \frac{1}{\sqrt{2}} \sqrt{ \ln\pare{\frac{2}{\delta}} +  \frac{\alpha}{\alpha - 1} \, \ln(n) +  D_{\alpha}(P_n, P_0)} \, D_{\zcp}\Big(P_n, P_0; \, \sqrt{\frac{n}{2\delta^2}} \Big)
    + \ln\left(1 + \frac{1}{\delta}\right) + 2 + \delta~.
\end{align*}
Now we make this bound to hold uniformly over $n \in \mathbb{N} \setminus \{1\}$.
In particular, denoting by $\mathrm{bound}(\delta)$ the r.h.s.\ of the inequality in the above,
we have that $\Pr\{\int \ln \Wealth_n(\th) \diff P_n(\th) > \mathrm{bound}(\delta)\} \leq 2 \delta$
(note that $2 \delta$ comes by applying a union bound since we used \Cref{prop:P-lnW-t} and Ville's inequality).
Now, to make this bound uniform over $n$ we apply a union bound over a set $n \in \bigcup_{i=2}^{\infty} [i] = \mathbb{N} \setminus \{1\}$, that is
\begin{align*}
  &\Pr\left\{\exists n \in \mathbb{N} \setminus \{1\} ~:~ \int \ln \Wealth_n(\th) \diff P_n(\th) > \mathrm{bound}\left(\frac{\delta}{n(n+1)} \right)\right\}\\
  &\quad\leq
  \sum_{n \geq 2} \Pr\left\{\int \ln \Wealth_n(\th) \diff P_n(\th) > \mathrm{bound}\left( \frac{\delta}{n(n+1)}  \right)\right\}
   \leq \sum_{n \in \mathbb{N} \setminus \{1\}} \frac{2 \delta}{n (n+1)} \leq \delta~.
\end{align*}
Finally we apply a regret bound (\Cref{thm:regret_up}) to lower bound $\int \ln \Wealth_n(\th) \diff P_n(\th)$ with  $\int \ln \Wealth_n^*(\th) \diff P_n(\th)$ (the regret $\ln(\sqrt{2(n+1)})$ appears at the r.h.s.).
Eventually we get
\begin{align*}
  &\int \ln \Wealth^*(\th) \dP_n(\th)\\
  &\quad\leq \frac{1}{\sqrt{2}} \sqrt{ \ln\pare{\frac{2 n (n+1)}{\delta}} +  \frac{\alpha}{\alpha - 1} \, \ln(n) +  D_{\alpha}(P_n, P_0)} \, D_{\zcp}\Big(P_n, P_0; \, \sqrt{\frac{n (n (n+1))^2}{2\delta^2}} \Big)\\
  &\qquad+ \ln\left(1 + \frac{n (n+1)}{\delta}\right) + \ln(\sqrt{2(n+1)}) + 2 + \frac{\delta}{n (n+1)}\\
  &\quad\leq \frac{1}{\sqrt{2}} \sqrt{ \ln\pare{\frac{4 n^2}{\delta}} +  \frac{\alpha}{\alpha - 1} \, \ln(n) +  D_{\alpha}(P_n, P_0)} \, D_{\zcp}\Big(P_n, P_0; \, \frac{\sqrt{2} \, n^{2.5}}{\delta} \Big)\\
    &\qquad + \ln\pare{2 e^2 \sqrt{n} \, \Big(1 + \frac{4 n^2}{\delta}\Big)} + \frac{\delta}{n (n+1)}~.
\end{align*}
The proof is now complete.
\jmlrQED \clearpage
\section{Proof of \Cref{prop:thm:zcp-log-wealth-asymptotics}: Asymptotic behaviour of expected log-wealth bound}
\label{sec:zcp-log-wealth-asymptotics}

Fix $P_0 \in \cM_1(\Th)$.
Let $\cK = \bigcup_{n=1}^{\infty} \cK(\cX^n, \Th)$.
Consider \Cref{thm:zcp-log-wealth-renyi} with $\delta = 1/n^2$ and $\alpha = \alpha_n := 1 + \frac{1}{\ln(n)}$:
\begin{align*}
  &\int \ln \Wealth^*(\th) \dP_n(\th)\\
  &\quad \leq\frac{1}{\sqrt{2}} \sqrt{ \ln(2 n^4) +  \ln(n) \ln(e n) + D_{\alpha_n}(P, P_0)} \, D_{\zcp}\Big(P, P_0; \, \sqrt{2} \, n^{4.5}\Big) 
  \\&\qquad+ \ln\pare{\sqrt{n} \, \Big(1 + 4 n^4\Big)} + \ln(2 e^2) + \frac{1}{n^3(n+1)}
  \\&\quad \leq\sqrt{ \ln(2 n^4) +  \ln(n) \ln(e n) + D_{\alpha_n}(P, P_0)}\pare{\fr{D_\ZCP(P, P_0;1)}{\sqrt{2}} + \sqrt{2\ln(2 + 2\sqrt{2} n^{4.5})} D_{\TV}(P, P_0)}
  \\&\qquad+ \ln\pare{\sqrt{n} \, \Big(1 + 4 n^4\Big)} + \ln(2e^3)
\end{align*}
where we applied \Cref{lem:int-to-zcp} to get the second inequality. 

Thus, abbreviating $F_n(P) = \int \ln W_n^*(\th) \diff P(\th)$ and a right-hand side in the above by $B_n(P, P_0)$ we have that
\begin{align*}
  \Pr\pare{\exists P \in \cK ~:~ F_n(P) > B_n(P, P_0)} \leq \frac{1}{n^2}
\end{align*}
Let $L(n) = \sqrt{2\ln(n) \ln(e n) \ln(2 + 2\sqrt{2} \, n^{4.5})}$ and note that
Noting that $L(n)$ is a dominating log-term in $B_n(P, P_0)$ for $n \geq 25$, and so using subadditivity of $\sqrt{\cdot}$ and some basic calculus gives
\begin{align*}
  \frac{B_n(P,P_0)}{L(n)}
  &\leq
  \underbrace{
    2 + \big(2 + \sqrt{D_{\alpha_n}(P, P_0)}\big) (D_\ZCP(P,P_0;1) + D_{\TV}(P, P_0))
  }_{=: A_n(P, P_0)}
\end{align*}
So,
\begin{align*}
  &\Pr\pare{\exists P \in \cK ~:~ \frac{F_n(P)}{L(n) A_n(P, P_0)} > 1}\\
  &\quad\leq
  \Pr\pare{\exists P \in \cK ~:~ \frac{F_n(P)}{L(n) \, A_n(P, P_0)} > \frac{B_n(P, P_0)}{L(n) \, A_n(P, P_0)}} \leq \frac{1}{n^2}~.
\end{align*}
Now verify that
\begin{align*}
  \sum_{n=25}^{\infty} \Pr\pare{\exists P \in \cK ~:~ \frac{F_n(P)}{L(n) A_n(P, P_0)} > 1} < \infty
\end{align*}
and by the Borel-Cantelli lemma
\begin{align*}
  \Pr\pare{ \cap_{n=1}^\infty \cup_{m=n}^\infty \cup_{P\in \cK}\cbr{\frac{F_m(P)}{L(m) \, A_m(P, P_0)} > 1 }} = 0~.
\end{align*}
Note that we have \begin{align*}
  \cup_{P\in \cK} \cap_{n=1}^\infty \cup_{m=n}^\infty \cbr[2]{\frac{F_m(P)}{L(m) \, A_m(P, P_0)} > 1 }
  \\\implies 
  \cap_{n=1}^\infty \cup_{m=n}^\infty \cup_{P\in \cK}\cbr[2]{\frac{F_m(P)}{L(m) \, A_m(P, P_0)} > 1 }~.
\end{align*}
Thus,
\begin{align*}
  \Pr\pare{\exists P \in \cK ~:~ \frac{F_n(P)}{L(n) \, A_n(P, P_0)} > 1 \quad i.o.} = 0
\end{align*}
which implies that
\begin{align*}
  \Pr\pare{\forall P \in \cK ~:~ \limsup_{n \to \infty} \ \frac{F_n(P)}{L(n) \, A_n(P, P_0)} \leq 1} = 1
\end{align*}
Now consider the property of $\limsup$ which states that for bounded real sequences $(a_n)_{n\geq 1}, (b_n)_{n\geq 1}$, $\limsup \  a_n b_n = b \limsup \ a_n$ whenever $\lim_{n \to \infty} b_n = b$.
Assuming for now that $(A_n(P, P_0))_n$ is bounded
we have
\begin{align*}
  \limsup_{n \to \infty} \ \frac{F_n(P)}{L(n) \, A_n(P, P_0)}
  =
  \pare{\frac{1}{\lim_{n\to \infty} \ A_n(P, P_0)}}\limsup_{n \to \infty} \ \frac{F_n(P)}{L(n)}~
\end{align*}
which means that
\begin{align*}
  \Pr\pare{\forall P \in \cK ~:~ \limsup_{n \to \infty} \ \frac{F_n(P)}{L(n)} \leq \lim_{n \to \infty} \ A_n(P, P_0)} = 1~.
\end{align*}
Finally, we look at the limit of $A_n$.
Since $P$ is absolutely continuous with respect to $P_0$, $D_{\KL}(P, P_0) < \infty$ and the same hold for $D_{\zcp}$.
Using the fact that $\lim_{n \to \infty} \ D_{\alpha_n}(P, P_0) = \lim_{\alpha \to 1} \  D_{\alpha}(P, P_0) = D_{\KL}(P, P_0)$, we have
\begin{align*}
  \lim_{n \to \infty} A_n =
  2 + \big(2 + \sqrt{D_{\KL}(P, P_0)}\big) (D_\ZCP(P,P_0;1) + D_{\TV}(P, P_0))~.
\end{align*}
The proof is now complete.
\jmlrQED

\clearpage
\section{Other omitted proofs}

\subsection{Proof of \Cref{lem:int-to-zcp}}
\label{sec:int-to-zcp-proof}
  The proof relies on the following lemma:
  
\begin{lemma}
  For every $x \in\RR$ and $c \ge 0$, we have
  \begin{align*}
    \ln(1 + cx^2) \le \ln(1 + x^2) + \ln(2 + 2c)
  \end{align*}
\end{lemma}
\begin{proof}
  If $cx^2 \ge 1$, then we have
  \begin{align*}
    \ln(1 + cx^2) \le \ln(2cx^2) = \ln(x^2) + \ln(2c) \le \ln(1 + x^2) + \ln(2c)~.
  \end{align*}
  Otherwise, we have
  \begin{align*}
    \ln(1 + cx^2) \le \ln(2) \le \ln(1 + x^2) + \ln(2)~.
  \end{align*}
  In either case, we have
  \begin{align*}
    \ln(1 + cx^2) 
    \le \ln(1 + x^2) + \ln(2 \vee 2c)
    \le \ln(1 + x^2) + \ln(2 + 2c)~.
  \end{align*}
\end{proof}
Using the lemma above, we obtain
\begin{align*}
  &D_{\zcp}(P, Q; c)
  \\&\quad=\int \abs{\fr{\diff P}{\diff Q}(\th) - 1} \sqrt{\ln\pare{1 + c^2 \, \pare{\fr{\diff P}{\diff Q}(\th) - 1}^2 }} \diff Q(\th)
  \\&\quad\leq
  \int \abs{\fr{\diff P}{\diff Q}(\th) - 1} \sqrt{\ln\pare{1 + \abs{\fr{\diff P}{\diff Q}(\th) - 1}^2 }} \diff Q(\th)
  +
  \sqrt{\ln(2 + 2c)} \, \int \abs{\fr{\diff P}{\diff Q}(\th) - 1} \diff Q(\th)
  \\&\quad=  \, D_\ZCP(P,Q; 1) + 2\sqrt{\ln(2 + 2c)} \, D_{\TV}(P, Q)~.
\end{align*}
\jmlrQED
\vfill
\subsection{Proof of \Cref{prop:Dh-to-tv-kl}}
\label{sec:Dh-to-tv-kl-proof}

One can see that  $\forall x \in \RR, \ln(1 + x^2) \le \ln(1 + 2|x|  + x^2) \le \ln( (1 + |x|)^2) = 2 \ln(1 + |x|)$.
Using this, 
\begin{align*}
  D_\ZCP(P,Q; 1)
  &= \int_{\Th} \envert{\frac{\diff P(\th)}{\diff Q(\th)} - 1} \sqrt{\ln\del{1 +
      \envert{\frac{\diff P(\th)}{\diff Q(\th)} - 1}^2 }} \diff Q(\th)
\\&\le \sqrt{2} \int_{\Th} \envert{\frac{\diff P(\th)}{\diff Q(\th)} - 1} \sqrt{\ln\del{1 +
      \envert{\frac{\diff P(\th)}{\diff Q(\th)} - 1} }} \diff Q(\th)
\\&\stackrel{(a)}{\le} \sqrt{2\int_{\Th} \envert{\frac{\diff P(\th)}{\diff Q(\th)} - 1} \diff Q(\th) \cd \int_{\Th}{\envert{\frac{\diff P(\th)}{\diff Q(\th)} - 1}\ln\del{1 +
      \envert{\frac{\diff P(\th)}{\diff Q(\th)} - 1}}} \diff Q(\th)}
\\&= \sqrt{2\int_{\Th} \envert{\frac{\diff P(\th)}{\diff Q(\th)} - 1} \diff Q(\th) \cd D_{f_1}(P,Q)} \tag{$f_1(x) = |x-1| \ln(1+|x-1|)$ }
\\&\stackrel{(b)}{\le} \sqrt{2\int_{\Th} \envert{\frac{\diff P(\th)}{\diff Q(\th)} - 1} \diff Q(\th) \cd 2D_{f_2}(P,Q)} \tag{$f_2(x) = 1 - x + x\ln(x)$}
\\&= \sqrt{4 D_\TV(P,Q) \cd 2D_{\KL}(P,Q)},
\end{align*}
where $(a)$ follows by Cauchy-Schwartz inequality and $(b)$ is by \citet[Lemma C.1]{zhang2022optimal}.
\jmlrQED
\subsection{Proof of \Cref{prop:gaussian-instance}}
\label{sec:gaussian-instance-proof}
\textbf{Case 1}: From \citep{nielsen2018guaranteed} we have that the TV distance is bounded as $D_{\TV}(P,Q)\leq\frac{1}{2}\cdot 2p=p$.
On the other hand, from the convexity of $\KL$ divergence, we have that 
\begin{align*}
  D_{\KL}(P,Q)
  \le p D_{\KL}(A,B)+(1-p) D_{\KL}(B,B) 
  =p D_{\KL}(A,B)~.
\end{align*}
Now, for Gaussians we have that
\[
D_{\KL}(A,B)
=\ln \frac{\sigma_2}{\sigma_1}+\frac{\sigma_1^2}{2 \sigma_2^2}-\frac{1}{2}
\]
and in particular
choosing $\sigma_2=p \sigma_1$, $D_{\KL}(A, B) = \ln(p)+\frac{1}{2p^2}-\frac{1}{2}$, which in turn implies that
$D_{\KL}(P,Q) \leq p D_{\KL}(A, B) = p (\ln p+\frac{1}{2p^2}-\frac{1}{2}) \le \frac{1}{2p}$.
Thus $D_{\TV}(P,Q) \cdot D_{\KL}(P,Q) \le \frac{1}{2}$.
On the other hand, we have that
\begin{align*}
D_{\KL}(P,Q) 
&= \int \dP \ln \frac{p \dif A +(1-p) \dif B}{ \dif B}\\
&=\int \dP \ln\left(\frac{p \dif A}{\dif B} +(1-p)\right)\\
&= \ln (1-p) +\int \dP \ln\left(\frac{p \dif A}{(1-p)\dif B} +1\right)\\
&= \ln (1-p) +\int (p\dif A+ (1-p) \dif B) \ln\left(\frac{p \dif A}{(1-p)\dif B} +1\right)\\
& \geq \ln (1-p) +\int p\dif A \ln\left(\frac{p \dif A}{(1-p)\dif B} +1\right)\\
& \geq \ln (1-p) +\int p\dif A \ln\frac{p \dif A}{(1-p)\dif B}\\
  & = \ln (1-p) + p D_{\KL}(A,B) + p \ln \frac{p}{1-p}\\
  &= \ln (1-p) + p \ln(p)+\frac{1}{2p}-\frac{p}{2} + p \ln \frac{p}{1-p}\\
&\geq \frac{1}{2p}-1.3,
\end{align*}
where $-1.3$ comes from the  minimization of $p \mapsto \ln (1-p) + p \ln(p) -\frac{p}{2} + p \ln \frac{p}{1-p}$ over $p \in [0,1]$.

\textbf{Case 2:} Choosing $\sigma_2=p^{3/4}\sigma_1$, $D_{\TV}(P,Q)\leq p$ and $D_{\KL}(P,Q) \leq p D_{\KL}(A, B) = p(\ln \frac{\sigma_2}{\sigma_1}+\frac{\sigma_1^2}{2 \sigma_2^2}-\frac{1}{2})=p(\frac{3}{4}\ln p+\frac{1}{2p^{3/2}}-\frac{1}{2})\le \frac{1}{2\sqrt{p}}$. Thus, $D_{\KL}(P,Q) \cdot \sqrt{D_{\TV}(P,Q)}\le \frac{1}{2}$. On the other hand, reasoning as above, we have
\begin{align*}
D_{\KL}(P,Q) 
  & \ge  \ln (1-p) + p D_{\KL}(A,B) + p \ln \frac{p}{1-p}\\
  &= \ln (1-p) + p \left(\frac{3}{4}\ln p+\frac{1}{2p^{3/2}}-\frac{1}{2}\right)+p \ln \frac{p}{1-p},\\
  & \ge \frac{1}{2\sqrt{p}}-1.22,
\end{align*}
where $-1.22$ comes from the  minimization of $\ln (1-p) + p (\frac{3}{4}\ln p-\frac{1}{2})+p \ln \frac{p}{1-p}$ over $p \in [0,1]$.
\jmlrQED
\subsection{Proof of empirical Bernstein inequalities (\Cref{cor:empiricalbernstein})}
\label{sec:empiricalbernstein-proof}
To show this corollary we follow \citet[Proposition 4]{jang2023tighter}.
Abbreviate $\hat \mu_{\th} = \frac1n \sum_{i=1}^n f(\th, X_i)$, and so $\Delta_n(\th) + \hat{\mu}_{\theta}  \in [0,1]$.
Then, we have
\begin{align*}
  \ln \Wealth_n^*(\th)
  \geq
  \max_{\beta \in [-1,1]} \ \sum_{i=1}^n \ln\pare{1 + \beta \pare{f(\theta,X_i) - (\hat{\mu}_{\theta} + \Delta(\th))}},
\end{align*}
and applying Jensen's inequality
\begin{align}
  \label{eq:bernproof1}
  \int \ln \Wealth_n^*(\th) \diff P_n
  \geq
  \max_{\beta \in [-1,1]} \ \int \sum_{i=1}^n \ln\pare{1 + \beta \pare{f(\theta,X_i) - (\hat{\mu}_{\theta} + \Delta(\th))}} \diff P_n~.
\end{align}
We relax the above by taking a lower bound of \citep[Eq. 4.12]{fan2015exponential} which shows that
for any $|x| \leq 1$ and $|\beta| \leq 1$,
\begin{align}
  \label{eq:logineq}
  \ln(1 + \beta x) \geq \beta x + \big(\ln(1 - |\beta|) + |\beta|\big) x^2~.
\end{align}
Then, combined with the following lemma:
\begin{lemma}[{\citet[Lemma 5]{orabona2023tight}}]
  \label{lem:empiricalbernsteintechnical}
  Let $f(\beta) = a \beta + b \big(\ln(1 - |\beta|) + |\beta|\big)$ for some $a \in \reals, b \geq 0$.
  Then, $\max_{\beta \in [-1,1]} f(\beta) \geq \frac{a^2}{(4/3) |a| + 2 b}$.
\end{lemma}
we get a chain of inequalities:
\begin{align*}
  \int \ln \Wealth_n^*(\th) \diff P_n
  &\stackrel{(a)}{\geq}
    \beta \int \sum_{i=1}^n \pare{f(\theta,X_i) - (\hat{\mu}_{\theta} + \Delta_n(\th))} \diff P_n\\
  &\quad+ \big(\ln(1 - |\beta|) + |\beta|\big) \int \sum_{i=1}^n \pare{f(\theta,X_i) - (\hat{\mu}_{\theta} + \Delta_n(\th))}^2 \diff P_n\\
  &=
    - n \beta \int \Delta_n(\th) \diff P_n\\
  &\quad+ \big(\ln(1 - |\beta|) + |\beta|\big) \pare{
    \int \sum_{i=1}^n (f(\theta,X_i) - \hat{\mu}_{\theta})^2 \diff P_n + n \int \Delta_n(\th)^2 \diff P_n
    }\\
  &\stackrel{(b)}{\geq}
    - n \beta \int \Delta_n(\th) \diff P_n\\
  &\quad + \big(\ln(1 - |\beta|) + |\beta|\big) \pare{
    \int \sum_{i=1}^n (f(\theta,X_i) - \hat{\mu}_{\theta})^2 \diff P_n + n \pare{\int \Delta_n(\th) \diff P_n}^2
    }\\
  &\stackrel{(c)}{\geq}
    \frac{n^2 \pare{\int \Delta_n(\th) \diff P_n}^2}{(4/3) n \left| \int \Delta_n(\th) \diff P_n \right|
    +    
    2 \int \sum_{i=1}^n (f(\theta,X_i) - \hat{\mu}_{\theta})^2 \diff P_n + 2 n \pare{\int \Delta(\th) \diff P_n}^2
    }.
\end{align*}
Here, step $(a)$ comes from \cref{eq:bernproof1,eq:logineq}, whereas $(b)$ comes from Jensen's inequality, and step $(c)$ is due to application of \Cref{lem:empiricalbernsteintechnical}.
At this point, the first result of \Cref{cor:empiricalbernstein} comes from the above combined with the fact that $|\Delta_n| \leq 1$ and by using \cref{eq:asymptotic-bound}.

Now, to state the second result of \Cref{cor:empiricalbernstein} we use a PAC-Bayes bound of \Cref{thm:zcp-log-wealth-renyi} to have
\begin{align*}
  n^2 \pare{\int \Delta_n(\th) \diff P_n}^2
  \leq
  n \, \mathrm{Comp}_n(\alpha) \,
  \pare{ \frac43 \left| \int \Delta_n(\th) \diff P_n \right|
    +    
  2 \hat{V}(P_n) + 2 \pare{\int \Delta_n(\th) \diff P_n}^2}~.
\end{align*}
Solving the above for $\int \Delta_n(\th) \diff P_n$, using subadditivity of square root, and relaxing some numerical constants
we get
\[
  \abs{\int \Delta_n(\th) \diff P_n}
  \leq
  \frac{\sqrt{2 \, \mathrm{Comp}_n(\alpha) \, \hat{V}(P_n)}}{(\sqrt{n} - (2/\sqrt{n}) \, \mathrm{Comp}_n(\alpha))_+}
  +
  \frac{2 \mathrm{Comp}_n(\alpha)}{\pare{n - 2 \, \mathrm{Comp}_n(\alpha)}_+}~.
\]
The proof of the asymptotic version is immediate from the proof of \Cref{prop:thm:zcp-log-wealth-asymptotics}.
\jmlrQED
\subsection{Proof of \Cref{lemma:wealth_lower_bound}}
  We have that
  \begin{align*}
    \exp\pare{\frac{\pare{\sum_{i=1}^n c_i}^2}{4 n}}
    &=
      \max_{\beta \in [-1/2, 1/2]} \exp\pare{\beta \sum_{i=1}^n c_i - \beta^2 n} \\
    &\leq
      \max_{\beta \in [-1/2, 1/2]} \exp\pare{\beta \sum_{i=1}^n c_i - \beta^2 \sum_i c_i^2}\\
    &\leq
      \max_{\beta \in [-1/2, 1/2]} \exp\pare{\sum_{i=1}^n \ln(1 + \beta c_i)},
  \end{align*}
  where we use the elementary inequality $\ln(1+x) \geq x - x^2$ for $|x| \leq 1/2$.
\jmlrQED

\end{document}